\documentclass[twoside]{article}

\usepackage[accepted]{aistats2025}
\usepackage{xr}
\usepackage{xspace} 
\usepackage{microtype} 
\usepackage{graphicx}
\usepackage{subfigure}
\usepackage{booktabs} 
\usepackage{amsfonts} 
\usepackage{amsmath,amsthm} 
\usepackage{amssymb,amsmath,amsthm}
\newtheorem{theorem}{Theorem}
\newtheorem{lemma}{Lemma}
\usepackage{mathtools} 

\usepackage{algorithm}
\usepackage{algorithmic} 
\usepackage{float} 
\restylefloat{table}
\usepackage{caption} 

\newcommand{\bbR}{\mathbb{R}} 
\newcommand{\bbE}{\mathbb{E}} 
\usepackage{stfloats} 
\usepackage{threeparttable} 
\usepackage{comment} 
\usepackage{xcolor} 
\usepackage{wrapfig} 
\usepackage{enumitem}
\usepackage{caption} 
\usepackage{wrapfig} 
\usepackage{soul}

\allowdisplaybreaks 

\newtheorem*{claim*}{Claim} 
\newtheorem*{theorem*}{Theorem} 
\newtheorem*{corollary*}{Corollary}
\newtheorem{corollary}{Corollary}
\newtheorem{proposition}{Proposition}

\usepackage{hyperref} 
\usepackage{natbib}
\usepackage{capt-of}
\usepackage[]{graphicx}
\usepackage[capitalize,noabbrev]{cleveref}

\usepackage[textsize=tiny]{todonotes}

\graphicspath{{./figs/}}

\usepackage[textsize=tiny]{todonotes}

\begin{document}
\runningauthor{Kaan Ozkara, Tao Yu, Youngsuk Park}

\twocolumn[

\aistatstitle{Stochastic Rounding for LLM Training: Theory and Practice}

\aistatsauthor{ Kaan Ozkara \textsuperscript\textdagger \And Tao Yu \And  Youngsuk Park }


\aistatsaddress{ UCLA \\ kaan@ucla.edu \And  AWS AI \\ taou@amazon.com \And AWS AI \\ pyoungsu@amazon.com } ]

\begin{abstract}
  As the parameters of Large Language Models (LLMs) have scaled to hundreds of billions, the demand for efficient training methods—balancing faster computation and reduced memory usage without sacrificing accuracy—has become more critical than ever. In recent years, various mixed precision strategies, which involve different precision levels for optimization components, have been proposed to increase training speed with minimal accuracy degradation. However, these strategies often require manual adjustments and lack theoretical justification. In this work, we leverage stochastic rounding (SR) to address numerical errors of training with low-precision representation. We provide theoretical analyses of implicit regularization and convergence under the Adam optimizer when SR is utilized. With the insights from these analyses, we extend previous BF16 + SR strategy to be used in distributed settings, enhancing the stability and performance for large scale training. Empirical results from pre-training models with up to 6.7B parameters, for the first time, demonstrate that our BF16 with SR strategy outperforms (BF16, FP32) mixed precision strategies, achieving better validation perplexity, up to 1.54× higher throughput, and 30\% less memory usage. 
\end{abstract}

\section{Introduction}

With the increased scale of LLMs, mixed precision (MP) training strategies \citep{micikevicius2018mixedprecisiontraining} has become the \textit{de facto} training strategy. Compared to FP32 training, MP vastly increases the throughput without degrading the accuracy of the model by delegating some or most operations to 16-bit depending on the strategy used. On the other hand, vanilla BF16 training, which is even more efficient and faster than MP as every tensor and operations are in BF16 (with nearest rounding behavior for FP operations), has been shown to cause significant degradation in accuracy (see \cref{fig:350M val} and \citep{zamirai2021revisitingbfloat16training,rae2022scalinglanguagemodelsmethods}).  


Previous efforts \citep{zamirai2021revisitingbfloat16training,rae2022scalinglanguagemodelsmethods} have proposed changing the default nearest rounding (NR) behavior of some operators to stochastic rounding (SR) \citep{croci2022stochastic}, an unbiased estimator for quantization. Utilizing SR in the model update step of BF16 training improves the trained model performance compared to vanilla BF16 approach. Recent accelerators such as AWS Trainium instances enable SR for all downcasting operations (e.g. in GEMMs) in hardware level \citep{fan2024hlat,trainium, muhamed2023training}. However, there is still a clear performance gap between BF16+SR and MP for LLM training (for instance a 7\% degradation in validation perplexity for BERT-base \citep{zamirai2021revisitingbfloat16training}, and 2\% degradation in validation loss for a 420M parameter decoder model \citep{rae2022scalinglanguagemodelsmethods}). These works have used the same hyper-parameters for MP and SR training for a direct comparison.

\begin{figure*}[t]
    \begin{minipage}[t]{.45\linewidth}
    \vspace{20pt}
    \centering
         \begin{tabular}{lccc} \hline 
         & BF16+SR & (BF16,FP32) MP  & BF16  \\ \hline
        Accuracy & {\color{teal}$\boldsymbol{\uparrow}$}* & $\uparrow$ & {\color{red}$\downarrow$} \\
        Throughput & {\color{teal}$\boldsymbol{\uparrow}$} \ &  $\uparrow$ &  {\color{teal}$\boldsymbol{\uparrow}$}  \\
        Memory eff. &  {\color{teal}$\boldsymbol{\uparrow}$} \ & {\color{red}$\downarrow$}  & {\color{teal}$\boldsymbol{\uparrow}$} \\
        Robustness & {\color{teal}$\boldsymbol{\uparrow}$}* & {\color{red}$\downarrow$}  & {\color{red}$\downarrow$} \\ \hline
    \end{tabular}
  \captionof{table}*
      {
       \small{* denotes first shown in this work.}
      }
    \end{minipage}\hspace{35pt}
    \begin{minipage}[t]{.45\linewidth}
    \vspace{0pt}
        \centering
        \includegraphics[scale=0.4]{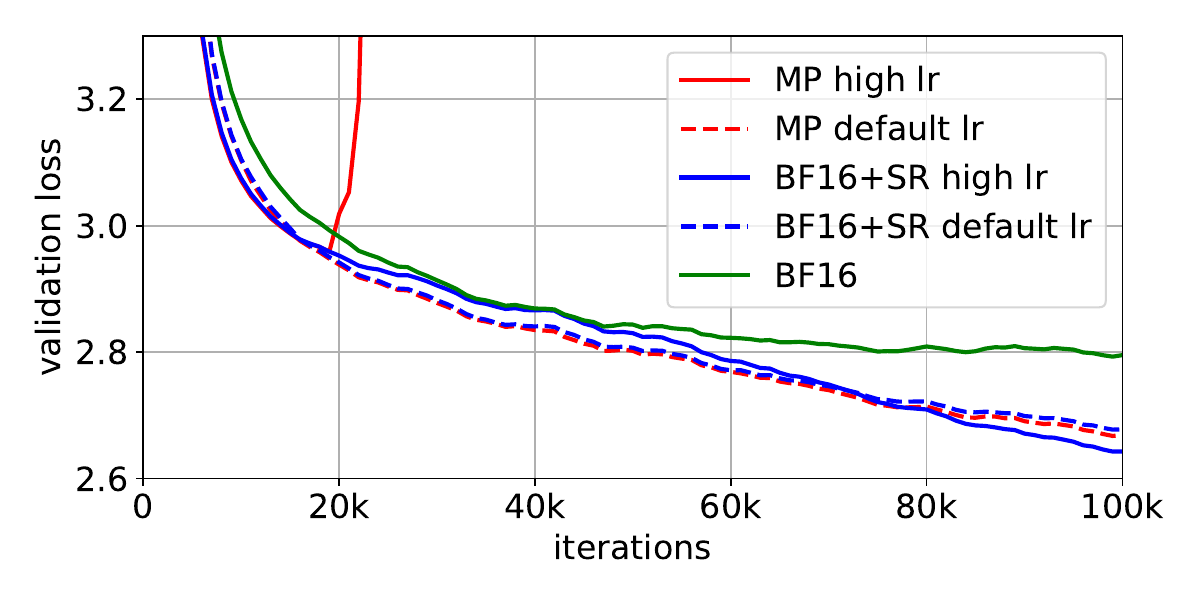}
    \end{minipage}
    \caption{(Left) Performance summary of competing methods. Our BF16+SR strategy shows better perplexity, faster throughput, less memory usage and more robustness towards high learning rates compared to the state of the art mixed precision methods. (Right) Validation losses of competing methods while pre-training GPT-2 (350M). When employed with a relatively high learning rate, training with SR outperforms mixed precision strategy. If the same high learning rate is used for the mixed precision training we observe divergent training, which indicates robustness of SR to higher learning rates. } \label{fig:350M val}
\end{figure*}

In this work, our primary goal is to understand the practical and theoretical properties of stochastic rounding (SR)
in training of LLMs. We hypothesize that 
as stochastic rounding provides an unbiased estimator, it leads to improved performance in training with less error accumulated. However, stochastic and nearest rounding's interactions with the widely adopted Adam optimizer \citep{KingmaB14} have not been thoroughly examined in the literature. This lack of understanding presents an opportunity to explore how different rounding methods affect convergence and overall training dynamics. By investigating these questions, we aim to provide valuable insights, both practically--by improving training efficiency and accuracy--and theoretically--by advancing the knowledge of optimizer behavior under varying rounding conditions.

We reveal that, 
SR shows a lower quantization error compared to NR in the convergence bound. Empirically, we show that BF16+SR strategy can outperform state of the art (BF16, FP32) MP training in terms of accuracy for training billion scale LLMs (up to 7B in our experiments), while retaining all the advantages of BF16 training. Our detailed contributions are as follows:

\begin{itemize}
    \item Theoretically, we analyze convergence properties of SR \& NR in Adam in a non-convex setting (applied to the update step), and conclude that the additional convergence error due to SR is strictly smaller than error when NR is used. With appropriate choice of hyper-parameters (e.g. high learning rate) quantization error due to SR can be subsumed by Adam's original convergence bound.  
    \item We demonstrate that training with SR implicitly regularizes the loss function by a quantization error penalty term. By examining training with small learning rates, we find that the penalty term may dominate and result in stagnation, which highlights the necessity of using higher learning rates for effective SR training.
    \item We propose a BF16 AdamW optimizer \citep{loshchilov2018decoupled}, with SR applied to the model update step. This optimizer for the first time empirically outperforms mixed precision training, with higher throughput and less memory usage in training billion parameter language models. 
\end{itemize}


\textbf{Outline. } In \cref{sec:mot and back}, first, we introduce the problem and necessary backgrounds on rounding options. Then, we motivate the use of stochastic rounding and examine the effect of learning rate. Next, in \cref{sec:algorithm}, we introduce the BF16-AdamW-SR algorithm utilizing stochastic rounding. In \cref{sec:theory}, we analyze the implicit regularization effect of SR, and compare the convergence behavior of AdamW under stochastic, and nearest rounding. Lastly, in \cref{sec:experiments}, we present our experimental results showcasing superior efficiency and efficacy of training with BF16+SR compared to mixed precision (BF16, FP32) and (nearest rounding) BF16 training schemes. Proofs of results, additional experiments, and details can be found in appendices.

\subsection{Related works}

\textbf{Stochastic rounding.} SR is an alternative rounding approach to the canonical nearest rounding (NR) where high precision number is quantized to the nearest low precision value. While NR has the best worst error guarantees, it is biased. SR on the other hand, is an unbiased quantizer, as the probability of quantization is inversely proportional to the distance to low precision values, which makes it more suitable for training ML models to reduce accumulating errors across iterations \cite{gupta2015deep}. 
In particular, \cite{connoly2021srprob} shows that SR has better probabilistic error bounds when summing small numbers, as in deep learning applications, compared to nearest rounding. \cite{li2017training,xia2023influencestochasticroundofferrors} analyze error terms due to SR in gradient descent and show that for simple machine learning tasks (e.g. regression) SR gives competitive performance. \cite{zaheer2018adaptive,rae2022scalinglanguagemodelsmethods} apply SR for training $\leq$1B scale LLMs and argue that there is a performance discrepancy compared to mixed precision training or FP32 training.

\textbf{Efficient LLM training.} Traditional machine learning training algorithms generally rely on 32-bit representations (FP32) to avoid perturbing the training performance. 
Recently, many works have been proposed to lower the training costs without accuracy drop using various data types, 
such as FP16 \citep{micikevicius2018mixedprecisiontraining}, BF16 \citep{kalamkar2019studybfloat16deeplearning}, FP8 \citep{wang2018training,sun2019hybrid} and so on.  \cite{micikevicius2018mixedprecisiontraining} realized that FP16 operations can be utilized by keeping a FP32 master copy for accumulating the updates without accuracy degradation. \cite{zamirai2021revisitingbfloat16training} showed that Kahan summation \citep{kahan1997mindless} could be utilized to account for numerical errors of BF16 computations. \cite{yu2024collage} used multiple-component floating-point to reduce errors during BF16 computations. Contrary to previous work, our method is able to match the performance of mixed precision training without introducing any auxiliary variables.

\section{Motivation and Background}\label{sec:mot and back}

In this work, we are interested in minimizing a non-convex loss function $F : \bbR^{d} \rightarrow \bbR$, and in optimization problems of the form
$$    \min_{x \in \mathbb{R}^d} F(x).
$$
We denote by $f : \bbR^{d} \rightarrow \bbR$ a random function computed on a mini-batch sampled from the underlying data distribution. Therefore, for any model parameters $x\in \bbR^{d}$, we have  $\bbE \left[f(x)\right] = F(x)$. We assume $F$ is differentiable and that $\bbE \left[\nabla f(x)\right] = \nabla F(x)$ for all $x\in \bbR^{d}$. We use $f_t$ or $f(x_t)$ to denote the random function evaluated with input batch sampled at time $t$.

\subsection{Stochastic and nearest rounding}

Here we define stochastic rounding and examine why it is more preferred to nearest rounding for training. Nearest (deterministic) rounding can be formalized as,
\begin{equation}
    Q(x):= \mathrm{sign}(x) \cdot \Delta_x \cdot \Big\lfloor \frac{x}{\Delta_x} + \frac{1}{2} \Big\rfloor,
\end{equation}

where $x$ is input with a higher precision compared to output (e.g. 32-bit and 16-bit representations), $\Delta_x = \lceil x \rceil - \lfloor x\rfloor$ is resolution, i.e. the distance between quantization levels, around $x$. Note, for non-uniform data types such as floating point representations the resolution depends on the magnitude of the input, which is denoted by the subscript $\Delta_x$. Nearest rounding is the default rounding mode for (IEEE) FP operations (such as summation/subtraction  and so on); typically, a higher precision buffer is used and then the result is nearest rounded. If the operands are in different precision, the lower precision is upcast and standard FP operations are followed afterwards. Next, we define the stochastic rounding operation,
\begin{equation}
    Q_{SR}(x):=
    \begin{cases}
       \lceil x \rceil , & \text{w.p.}\ \frac{x - \lfloor x \rfloor}{\Delta_x} \\
       \lfloor x\rfloor , & \text{w.p.}\ \frac{\lceil x \rceil - x}{\Delta_x}
    \end{cases}.
\end{equation}

The crucial property of SR is, contrary to nearest rounding, unbiasedness, $\bbE[Q_{SR}(x)]=x$. Unbiasedness property prevents stagnation (or imprecision), in expectation, particularly when two numbers of different magnitude is summed in the update step. As an example, consider the gradient step for the quantized model in the form $x_{t+1}=x_{t}+u$ where $u$ is the update that depends on learning rate and gradient information. When $u \ll x_{t}$, we have $Q(x_{t}+u) = x_{t}$; hence, stagnation occurs. However, with SR we have $\bbE [Q_{SR}(x_{t}+u)] = x_{t}+u$. 

\subsection{Effect of learning rate under SR}\label{sec:lr with sr}


\citet{zamirai2021revisitingbfloat16training,rae2022scalinglanguagemodelsmethods} have attempted at BF16 training with SR but found out there is a performance gap with mixed precision training. Here, we present toy examples where this gap becomes more pronounced with a small learning rate, suggesting the necessity of using a higher learning rate for SR training.
We model SR as a random walk and see that with small updates SR can take a long time to converge compared to higher precision training.


\begin{figure}[h]
    \centering
    \includegraphics[width=1.0\linewidth]{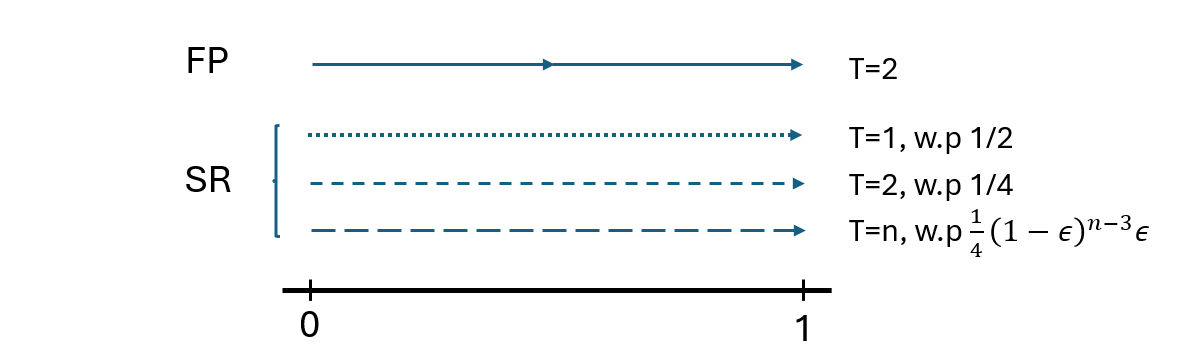}
    \caption{Depiction of updates when full precision and quantized stochastic rounding updates are used in the toy example.}
    \label{fig:random_walk}
\end{figure}

\textbf{Example. }Let $x_0 = 0$, we are interested in expected number of iterations $\bbE[T]$ where $T: x_T=1$ assuming 0 and 1 are consecutive quantization levels. Let the updates be $u_t=1/2$ for $t=0,1$ and $u_t = \epsilon$ for $t>2$. Note, with accurate full precision updates $x_2 = 1$, the next level is achieved in 2 steps. Whereas, for stochastic updates $T\leq 2$ w.p. $\frac{3}{4}$, and $T=n, n \geq 3$ w.p. $\frac{1}{4}(1-\epsilon)^{n-3}\epsilon$ (as seen in \cref{fig:random_walk}). Consequently, we have $\bbE[T]>2$ for $\epsilon<\frac{1}{2}$; and $\bbE[T] \to \infty$ for decaying updates $\epsilon \to 0$. 

\textbf{Example. } We also present a toy simulation example with a linear function. Consider a linear loss function,
\[
f(x) =
  \begin{cases}
   -(x - 2) + 1 & \text{for } x < 3, \\
   0 & \text{otherwise},
  \end{cases}
\]
iterated over with Adam optimizer using a decaying learning rate with minimum learning rate $1 \times 10^{-5}$. When we initialize from $x_0 = 2$, FP32 updates converge in 4600 steps; whereas the BF16 stochastic rounding updates converge in average in 7700 steps (plots provided in \cref{app:experiments}). And when the minimum learning rate is $5 \times 10^{-5}$ SR converges in $\sim$4600 steps.


The toy examples indicate that SR can still result in stagnation. This stagnation can be prevented when the updates are larger. To this end, in the next sections we will conclude, both theoretically and empirically, choosing a higher learning rate for SR training is critical for a competitive performance. A caveat is, in practice, Adam is known to diverge with higher learning rates; we find that SR alleviates this in \cref{sec:high lr} and \cref{sec:experiments}.


\subsection{Robustness of Adam with SR} \label{sec:high lr}
The examples so far implied that a higher learning rate is necessary for SR to be successfully employed. In practice, a higher learning rate results in unstable training of LLMs, even divergence. Such unstable training is associated with time-domain correlation of gradients in training with Adam optimizer. In particular, \citep{molybog2023adaminstability} shows that as time-domain correlation increases, a lower learning rate is required for training. In \cref{sec:experiments}, we observe that training with SR is more robust towards higher learning rates compared to BF16 nearest rounding training and (BF16-FP32) mixed precision training. We conjecture that additive noise while training with SR enables a learning rate-robust training by decorrelating the gradients. To formalize, we have the following proposition that demonstrates the decorrelation effect of SR in the simple toy setting in \cite{molybog2023adaminstability},
\begin{proposition}
    Assume that gradient estimated at dimension $i\in\{1,\ldots,d\}$ and $t\in[1,\ldots, T]$ is composed of two Bernoulli parts $g[i,t] = g^{(1)}[i] + g^{(2)}[i,t]$, where $g^{(1)}[i] \sim \rho \times \text{Bernoulli}(\frac{1}{2})$ and $g^{(2)}[i,t] \sim \text{Bernoulli}(\frac{1}{2})$ are independent. For simplicity, we model SR perturbation as an additional independent term $\xi[i,t] \sim \Delta \times \text{Bernoulli}(\frac{1}{2})$, where $\Delta$ denotes aggressiveness of quantization (e.g. resolution); and $g_{SR}[i,t] = g^{(1)}[i] + g^{(2)}[i,t] + \xi[i,t]$. Then correlation among gradients across time-domain while training with SR is $\mathrm{cor}_{SR} = \frac{\rho^2}{1+\rho^2+\Delta^2}$, and without SR is $\mathrm{cor} = \frac{\rho^2}{1+\rho^2}$.  
\end{proposition}
Clearly, $\mathrm{cor} > \mathrm{cor}_{SR}$ and $\Delta$ reduces the correlation, hence, increases stability. Note, in practice, $\xi[i,t]$ is dependent on the gradients but there is no clear correlation. Our result holds as long as $\xi$ is not negatively correlated with gradients. 

\section{Adam Optimizer with SR} \label{sec:algorithm}

\begin{algorithm} 
\caption{AdamW with Stochastic Rounding and Shared Randomness} \label{alg}
\begin{algorithmic}[1]
\REQUIRE Learning rate $\alpha_t$, beta parameters $(\beta_1, \beta_2)$, weight decay $\lambda$, reduced gradient $\nabla f(x_t)$ at iteration $t$, number of devices $M$, $r^m_t$ random states for each device, $r^{(opt)}_t$ random state for the optimizer.
\STATE $c_{1,0} \gets 1$, $c_{2,0} \gets 1$
\FOR{$t=1$ to $T$}
\FOR{$m=1$ to M}
    \STATE $g_{t} \gets \nabla f(x_t)$ ~~~~~~~~\#\# reduced gradient  
    \STATE $m_{t+1} \gets \beta_1 \cdot m_t + (1 - \beta_1) \cdot g_t$
    \STATE $v_{t+1} \gets \beta_2 \cdot v_t + (1 - \beta_2) \cdot g_t^2$
    \STATE $\hat{m}_{t+1} \gets \frac{m_{t+1}}{1 -( \beta_1)^t}$
    \STATE $\hat{v}_{t+1} \gets \sqrt{\frac{v_{t+1}}{1 - (\beta_2)^t}}$
    \STATE $r \gets r^{(opt)}_t$ \#\# set the same random state across ranks for optimizer
    \STATE $x_{t+1} \gets Q_{SR}\Big(x_t - \Big( \alpha_t \cdot \frac{\hat{m}_{t+1}}{\sqrt{\hat{v}_{t+1} + \epsilon}} + \alpha_t \cdot \lambda \cdot x_t \Big)\Big)$ \#\# $r^{(opt)}_t$ is updated during SR
    \STATE $r \gets r^{m}_t$ \#\# get the original random state
\ENDFOR
\ENDFOR
\RETURN $x_T$
\end{algorithmic}
\end{algorithm}




In this section, we present the AdamW algorithm with SR applied to the update step and shared randomness. \cref{alg} assumes a distributed setup that consists of $M$ model replicas. In our experiments, every tensor in \cref{alg} is in BF16 data type. At the start of every iteration, the devices receive reduced (aggregated) mini-batch gradients (line 4); Using the gradients, the optimizer states of Adam is updated (line 5,6). Before the update step, (line 9) we fork the random state such that every rank has the same randomness; otherwise, model replicas drift from each other, since same input in line 10 could be quantized to different model weights due to SR. As models diverge from each other, the computed/aggregated gradients deviates from its expectations,
which will disrupt the training. 

The model update (line 10) with SR is done by \emph{dithering} \citep{Schuchman1964}, which temporarily upcasts to FP32 and adds with a uniform FP32 noise to the mantissa bits, and finally shifting out the fractional part; this implementation is standard and practically has nearly no overhead on throughput (also simulates the hardware implementation e.g. \citep{trainium}). Since the update does not coincide with the forward pass, where the memory usage is the highest, temporary FP32 tensors do not affect the peak memory occupied during training. Consequently, BF16 training and BF16+SR training report similar throughput and memory usage. Note that mixed precision training typically uses FP32 optimizer states (lines 5,6) storage \& FP32 gradient communication (line 5), and keeps the master weights in FP32 (line 10). In contrast, our method adopts BF16 for all these variables. 


For our theoretical result in \cref{sec:theory}, we consider operations in lines 4,5,6 to be computed in full precision for a simple analysis. We ablate the precision of these operations in \cref{app:experiments}, where the performance difference with them in high precision vs. low precision is virtually non-existent. We further provide a proof sketch for the case where, for instance line 6 is also in low precision in \cref{app:convergence}.



\section{Theoretical Analysis} \label{sec:theory}

In this section, we start with implicit regularization effect of SR when combined with gradient descent. We utilize gradient descent for simplicity, general insights would translate to other first order optimization methods as well. Then, we move onto convergence properties of Adam when SR is used for parameter updates. 

\subsection{Implicit regularization under SR}
Gradient descent is assumed to discretize a continuous function $x(t)$ which is defined by the ODE $\dot{x}(t) = l(x(t)) := -\nabla F(x(t))$ (also called gradient flow). 
The reason for discretization is because gradient updates $x_{t+1}=x_t - \alpha \nabla F(x(t))$ are not computed continuously but only at discrete points. This results in modified gradient flow: $\dot{x}(t) = \tilde{l}(x(t))$, where $\tilde{l}(x(t)) := l(x(t)) + \alpha l_1(x(t)) + \alpha^2 l_2(x(t)) + \ldots$, here $l_i$ are error terms due to discrete approximation of gradient descent of the continuous function \cite{barrett2020implicit,smith2021on}. Note that as $\alpha \rightarrow 0$ we have continuous updates, i.e. $\tilde{l}(x(t)) \rightarrow l(x(t))$. Stochastic rounding (SR) introduces an additional error term such that the first order update is $\hat{l}(w) := l(w) +  {\xi}_{\alpha}(w)$. Hence, the modified flow becomes $\tilde{l}(x(t)) := \hat{l}(w) + \alpha \hat{l}_1(x(t)) + \alpha^2 \hat{l}_2(x(t)) + \mathcal{O}(\alpha^3)$ where $\hat{l}_i$ absorbs the perturbed higher order terms due to SR 
and $\xi_\alpha(x)$ defined as a Bernoulli random variable:

\begin{lemma}[Effective gradient perturbation] \label{lemma:xi}
\begin{equation*}
    \xi_\alpha(x)=
    \begin{cases}
       &\frac{1}{\alpha}\Big(\lceil x-\alpha\nabla F(x) \rceil - (x-\alpha \nabla F(x)) \Big) \\
       & \qquad \text{w.p.}\ \frac{x-\alpha\nabla F(x) - \lfloor x-\alpha\nabla F(x) \rfloor}{\Delta_x}, \\
       &\frac{1}{\alpha}\Big(\lfloor x-\alpha\nabla F(x)\rfloor - (x-\alpha\nabla F(x))\Big) \\
       & \qquad \text{w.p.}\ \frac{\lceil x-\alpha\nabla F(x) \rceil - (x-\alpha\nabla F(x))}{\Delta_x},
    \end{cases}
\end{equation*}

where $\Delta_x= \lceil x - \alpha \nabla F(x) \rceil - \lfloor x - \alpha \nabla F(x) \rfloor$ is the difference between quantization levels. Note that $\mathbb{E}[\xi_\alpha(x)]=0$ and variance is $\bbE[\|\xi_{\alpha}(x)\|^2] = \frac{1}{\alpha^2}\Big(\lceil (x-\alpha\nabla F(x) \rceil - (x-\alpha \nabla F(x)) \Big)^\top \Big(\lfloor (x-\alpha\nabla F(x)\rfloor - (x-\alpha\nabla F(x)\Big)$.
\end{lemma}



Now we present the implicit loss function that is minimized when gradient descent with perturbed updates due to SR is employed.
\begin{theorem} \label{thm:implicit reg}
    Let $F(x)$ be the loss function to be minimized, $\alpha$ be a constant learning rate, $\xi_{\alpha}(w_t)$ be random vector quantization error while doing updates with SR. Then during gradient descent with SR on the loss function, implicitly, the following expected modified loss, $F_{SR}(x)$, is being optimized:
    \begin{align*}
        F_{SR}(x) &:= F(x) + \frac{\alpha}{4} \| \nabla F(x) \|^2 + \frac{\alpha}{4}
        \bbE[\|\xi_\alpha(x)\|^2] 
    \end{align*}
    Furthermore, for $\alpha \rightarrow 0, \alpha > 0$ assuming that terms with first order dependence on $\alpha$ vanishes, we have that the second term vanished and from the third term we are left with $\frac{1}{4}\sum_{i=1}^d\Delta_{x_i}|\nabla F(x)_i| $.
\end{theorem}

\cref{thm:implicit reg} suggests that doing SR updates implicitly regularizes the loss function to minimize  the error due to quantization via SR. This implies the model is optimized in a quantization aware manner. Under no quantization error case, e.g. when $\Delta_x \rightarrow 0$, the third term vanishes and we obtain the original implicit loss function in \cite{barrett2020implicit}. Additionally, when $\alpha \rightarrow 0$ the quantization still has regularization effect that forces the gradients to be small. The strength of the effect depends on $\Delta_x$. Note that for a large $\Delta_x$ the quantization regularization may dominate and cause stagnation. Here, for simplicity we analyzed the case with gradient descent; the analysis could be extended to Adam optimizer following for instance \cite{cattaneo24on}. As shown in \cref{app:implicit reg}, with this modified flow, there is a $\mathcal{O}(\alpha^3)$ error between the true solution of gradient flow and discrete update (similar to \cite{barrett2020implicit}); For any biased rounding scheme, such as nearest rounding, there will be a $\mathcal{O}(\alpha)$ error difference between the true updates and modified updates at every iteration; consequently, at every iteration SR benefits from lower error in expectation (see \cref{app:implicit reg} for details).

\subsection{Convergence analysis of Adam with SR}

Our analysis is based on extending the Adam analysis in \cite{defossez2022a} with SR. We consider the case with no momentum, i.e. $\beta_1=0$, 
the extension to the case with momentum can be done easily as in \cite{defossez2022a}, 
effect of momentum is simply a multiplicative slow down term on all terms. For the analysis we make the following standard assumptions. 

\textbf{Assumptions.} (i) $F$ is bounded below by $F^*$, (ii)  bound on stochastic gradients, that is $\forall x \in \bbR^d, \ \|\nabla f(x)\|_{\infty} \leq R \text{ a.s. }$, (iii) $F$ is smooth, i.e.,  $\forall x,y \in \bbR^d, \quad \|\nabla F(x) - \nabla F(y) \|_2 \leq L\| x-y\|_2$.

\begin{theorem} [No momentum, $\beta_1=0$]  \label{thm:convergence} 
    Assuming access to full precision gradients, and learning rate $\alpha_t = \alpha\sqrt{\frac{1-\beta_2^t}{1-\beta_2}}$ with $\alpha > 0$, 
    when SR is used for the update step to obtain quantized weights, we have
    \begin{align*}
     & G_T \leq \frac{A}{T} + 
    \frac{2Rd}{T}\left( \frac{2R}{\sqrt{1-\beta_2}}+\frac{\alpha L}{1-\beta_2} \right) \left( T \ln\left(\frac{1}{\beta_2}\right) \right) \\
    & + \underbrace{ \frac{Rd \Delta L}{T\sqrt{1-\beta_2}} \left( \sqrt{T \ln\left(1 + \frac{R^2}{\epsilon(1-\beta_2)}\right)} + T \sqrt{\ln\left(\frac{1}{\beta_2}\right)} \right)}_{\text{quantization error}}.
\end{align*}
Here $\Delta = \max_{i} \Delta_{x_i}$, $G_T = \frac{1}{T}\sum_{t=1}^T \bbE \|\nabla F(x_t) \|^2$ and $A$ is a constant that depends on $d,R,\beta_2,\alpha$. 
\end{theorem} 

Note that the first two terms are in common with analysis in \cite{defossez2022a} where the second term, Adam's (non-vanishing) gap, is
    $
        2Rd
    \left( \frac{2R}{\sqrt{1-\beta_2}}+\frac{\alpha L}{1-\beta_2} \right)  \ln\left(\frac{1}{\beta_2}\right).
    $
And the third term, (non-vanishing) quantization error, is due to error introduced by SR as
    $
         \frac{Rd \Delta L}{\sqrt{1-\beta_2}} \sqrt{\ln\left(\frac{1}{\beta_2}\right)}.
    $
Now we examine when the quantization error (third term) becomes negligible or dominant over Adam's gap (second term). 
\begin{corollary} [Comparison to full precision Adam]
     Under $\Delta \ll \frac{\alpha}{\sqrt{1-\beta_2}}\sqrt{ln\left(\frac{1}{\beta_2} \right)}$ or $\Delta \ll \frac{R}{L}\sqrt{ln\left(\frac{1}{\beta_2} \right)}$, `$quantization \ error$' in \cref{thm:convergence} becomes much smaller than non-vanishing gap in Adam.
\end{corollary}


The effect of quantization error can be mitigated through hyper-parameters only and not solely dependent on the uncontrolled problem parameters (e.g. $L,R,d$). Note that with increased $\alpha$ the quantization error tends to be negligible, this is in parallel with our empirical observation that with high learning rate SR outperforms mixed precision training. Another observation we make is that error due to SR can dominate when $\beta_2 \rightarrow 1$ and $\alpha = \mathcal{O}(1/\sqrt{T})$. Although, in practice $\beta_2 < 1$ is used ($\beta_2 = 0.95$ commonly for GPT models), in theory $\beta_2 \rightarrow 1, \alpha = \mathcal{O}(1/\sqrt{T})$ diminishes the non-vanishing term in Adam error bound while resulting slow down for vanishing terms \citep{defossez2022a}, such cases will include a non-vanishing quantization error when SR is employed. Note that our theoretical observations is also in parallel with practical observations on performance drop with quantized training when $\beta_2 \approx 1$ in \citep{yu2024collage}.


\subsection{Convergence analysis of Adam with NR}

Here, we present the convergence analysis when NR is applied in the update step instead of SR. Note this corresponds to the mixed precision (MP) training where the update is in lower precision with NR but other parts (e.g. optimizer states) are in high precision. 

\begin{theorem}\label{thm:convergence nr}
 Under the same setting in \cref{thm:convergence}, Adam with NR gives the following convergence bound
 \vspace{-10pt}
    \begin{align*}
     & G_T \leq \frac{A}{T} + 
    \frac{2Rd}{T}\left( \frac{2R}{\sqrt{1-\beta_2}}+\frac{\alpha L}{1-\beta_2} \right) \left( T \ln\left(\frac{1}{\beta_2}\right) \right) \\
    & + \frac{{2}Rd \Delta L}{T\sqrt{1-\beta_2}} \left( \sqrt{T \ln\left(1 + \frac{R^2}{\epsilon(1-\beta_2)}\right)} + T \sqrt{\ln\left(\frac{1}{\beta_2}\right)} \right) \\
    & + {\frac{\sqrt{1-\beta_2}d(R\Delta + L\Delta^2)}{\alpha}}.
\end{align*}
\vspace{-10pt}
\end{theorem}

Note that nearest rounding results in a larger error bound due to (i) multiplicative term of $\times 2$ in the second line in \cref{thm:convergence nr} and (ii) additional error term in third line (both colored in red), indicating favorable convergence behaviour for SR training. The distinction is caused by the biasedness of nearest rounding, while $L\Delta^2$ is due to the variance (squared error), other error terms are the results of bias of the error term. This result emphasizes the detrimental effect of the accumulating bias resulting from nearest rounding. Focusing on the new non-vanishing term of nearest rounding, we observe that a large $\alpha$ can make the term negligible; however, in that case Adam's convergence errors start to hurt (second term in the bound). Therefore, getting rid of the new error term without perturbing Adam's original terms is not plausible. For nearest rounding, with a decaying learning rate, e.g. $\alpha=1/\sqrt{T}$, the error explodes unless $\beta_2 \rightarrow 1$.
\textbf{Remark. }
    The most related results are Theorem 1 from \cite{hou2018analysis} and Theorem 1,2 from \cite{hao2017trainingq}. Theorem 1 \citep{hou2018analysis} uses nearest rounding for weight quantization in Adam optimizer without momentum term (similar to our \cref{thm:convergence}) in a convex setting. To prevent error accumulations due to quantization bias, \citet{hou2018analysis} assumes that the diameter of weight space is bounded. As a result, even as $\Delta \rightarrow 0$, their bound includes a non-vanishing term that depends on diameter of the weight space which is not the case in our work.
 \cite{hao2017trainingq} uses SR with gradient descent in convex setting and derives non-vanishing term in convergence bound similar to ours. However, since the vanilla stochastic gradient descent does not include any non-vanishing terms, there are no cases in their result where the quantization error can be subsumed by the original full precision bound such as ours. Another related result is \cite{chen2021quantizedadamerrorfeedback} Theorem 3.2,  where the authors analyze the convergence of weight quantized Adam algorithm. Different than us, they assume $\ell_2$ bounds for gradient boundedness, which results non-vanishing term to implicitly depend on $d\sqrt{d}$ instead of $d$.


\begin{table*}[t]
\centering
    \begin{tabular}{lccccc}  
    Method & GPT-2(350M) & GPT-2(770M) & GPT-Neo(1.3B) & GPT-Neo(2.7B) & GPT-Neo(6.7B)\\ \hline
    BF16+SR (ours) &  \textbf{14.07} & \textbf{12.63} & \textbf{10.46} & \textbf{10.22} &\textbf{ 10.05} \\
    (BF16,FP32) MP & 14.45 & 12.83  & 10.48 & 10.31 & 10.11  \\
    BF16 & 16.38 & 15.28 &  11.81 & 11.67 & 11.64\\
  \end{tabular}
  \captionsetup{justification=centering}
  \captionof{table}
      {
        Validation perplexity of competing methods for GPT models. 
        \label{tab:val ppl}
      }
\end{table*}

\section{Experiments} \label{sec:experiments}

We showcase the empirical success of our SR training paradigm by comparing to state-of-the-art mixed precision strategies. As far as we are aware, our results are the first to outperform mixed precision training while using full BF16\footnote{except cross-entropy calculation for GPT-Neo, which is done in FP32 by default in code package.} training with SR \emph{without additional overhead and auxiliary tensors}. We first describe our experiment setting and then move on with the results.

\subsection{Settings}

\textbf{Baselines.} We compare with two widely-adopted mixed precision strategies: 
\begin{itemize}[nosep]
    \item O1-level: the native PyTorch's automatic mixed precision (\texttt{torch.amp}), which patches functions to internally carry out Tensor Core-friendly operations in BF16, and operations that benefit from additional precision in FP32. Because casts occur in functions, model weights remain FP32;
    \item O2-level: ``almost BF16" mixed precision proposed in \citep{micikevicius2018mixedprecisiontraining}, which maintains a high precision FP32 copy of the model as master weights and casts model weights (except batchnorm) and data to BF16 for forward \& backward passes. It uses FP32 for optimizer states, optimization and collective operations (e.g. all-reduce \& gradient accumulation) in distributed settings.
\end{itemize}
O1-level mixed precision has more operations in FP32, hence, is slower in terms of throughput but can result in lower perplexity compared to O2-level mixed precision. We further augment our experiments with a full BF16 lightweight training strategy with \emph{rounding-to-the-nearest} (as it is the default rounding mode for FP operations), where every tensor and operations, including optimizer states, are in BF16. Note, unless otherwise stated BF16 denotes BF16+default nearest rounding mode. Our SR strategy described in \cref{sec:algorithm} still keeps every tensor in BF16 but modifies the final update step (line 10 in \cref{alg}) in optimization to be done with SR. For the sake of simplicity none of the strategies utilize ZeRO (optimizer sharding) \citep{rajbhandari2020zero} or fused AdamW implementations.


\textbf{Models.} We conduct experiments using GPT-2 medium and large models, with 350M and 770M parameters respectively \citep{radford2019language}; and GPT-Neo \citep{gpt-neo} with 1.2B, 2.7B, 6.7B parameters. For GPT-2 experiments we follow nanoGPT \citep{Karpathy2022} and for GPT-Neo we use NeMo repository \footnote{\url{https://github.com/NVIDIA/NeMo}}.

\textbf{Hyper-parameters.} For each training strategy we use the same hyper-parameters except for the learning rate, which is individually tuned for each method. In particular, every competing method is trained with AdamW optimizer \citep{loshchilov2018decoupled} with default GPT training parameters $\beta_1 = 0.9, \beta_2 = 0.95, \epsilon=1e{-}8$, and weight decay $0.1$ (except GPT-2 350M model where we used $0.2$ for stability); our method, \cref{alg}, also uses the same AdamW parameters. {We defer more details (e.g. warmup schedule) to \cref{app:experiments}.}

\textbf{Training setup.} For GPT-2 (350M, 770M) we train, on 49B and 46B tokens for 100k iterations. For GPT-Neo (1.3B), we train on 40B tokens; for GPT-Neo (2.7B, 6.7B) we train on 20B tokens each for 20k iterations. More information on distributed training setup (e.g. global-micro batch sizes) can be found in \cref{app:experiments}.


\textbf{Learning rate.} We individually tune the learning rate of each training strategy. We observe that our SR method works better with a learning rate that is $2-4\times$ of the default learning rate recommended in O1 and O2 mixed precision strategy. Such a learning rate may cause unstable or divergent training mixed precision strategy (\cref{tab:lr}), and full BF16 training with rounding-to-the-nearest. The default learning rates for GPT-2 and GPT-Neo models are from \citep{liu2023sophia,gpt-neo} respectively. We tune the maximum learning rate and leave the minimum learning rate as suggested by default settings similar to \citep{ozkara24a}, see \cref{app:experiments} for more details. 
\begin{table}[h]
\centering
    \begin{tabular}{ccccc} 
    Model size & Default & MP & SR & MP converges  \\ \hline
    350M &  3 & 3 & 7 & {\color{red}$\times$}  \\
    770M & 2 & 2 & 7 & {\color{red}$\times$}  \\
    1.3B & 2 & 4 & 4 &  {\color{teal}\checkmark}\\
    2.7B & 1.6 & 4 & 5 & {\color{teal}\checkmark}   \\
    6.7B & 1.2 & 3.6 & 5.5 & {\color{red}$\times$} 
  \end{tabular}
  \captionof{table}
      {
        Default, tuned for mixed precision, tuned for SR learning rates ($\times 1e{-}4$); and whether mixed precision training \textbf{converges with SR's learning rate}. Note SR converges in all cases.%
        \label{tab:lr}
      }
\end{table}

\begin{table*}[t]
\centering
    \begin{tabular}{lccccc} 
    Method & GPT-2(350M) & GPT-2(770M) & GPT-Neo(1.3B) & GPT-Neo(2.7B) & GPT-Neo(6.7B)\\ \hline
    BF16+SR (ours) &  501 & 254 & 125 &  84 & 43 \\
    (BF16,FP32) MP & 469 & 224  & 105 & 57 & 28  \\ 
    Gain  & 1.07$\times$ & 1.14$\times$  & 1.19$\times$ & 1.47$\times$ & 1.54$\times$  \\
  \end{tabular}
  \captionof{table}
      {
        Throughput (sequences/second) of SR compared to O1 level mixed precision training (\texttt{torch.amp}).%
        \label{tab:throughput o1}
      }
\end{table*}

\begin{table*}[t]
\centering
    \begin{tabular}{lccccc} 
    Method & GPT-2(350M) & GPT-2(770M) & GPT-Neo(1.3B) & GPT-Neo(2.7B) & GPT-Neo(6.7B)\\ 
    BF16+SR (ours) &  186 & 208 & 184 &  171 & 184 \\
    (BF16,FP32) MP & 234 & 298  & 206 & 197 & 232  \\ \hline
    Tensor parallelism & 1 & 1 & 8& 8 & 8 \\ 
    \hline
    Memory savings  & $21\%$ & 30\%  & 11\% & 13\% & 21\%  \\
  \end{tabular}
  \captionof{table}
      {
        Memory consumption (GB) per node; SR compared to O1 level mixed precision training (\texttt{torch.amp}). 
        \label{tab:memory o1}
      }
\end{table*}

\begin{table}[h]
\centering
    \begin{tabular}{lccc}  
    Method \textbackslash GPT-Neo size & 1.3B & 2.7B & 6.7B\\ \hline
    BF16+SR (ours) & 135 &  75 & 32 \\
    (BF16,FP32) MP & 129 & 71 & 31  \\
    Gain  & 1.05$\times$ & 1.06$\times$  & 1.03$\times$ 
  \end{tabular}
  \captionof{table}
      {
        Throughput of SR compared to O2 level mixed precision training evaluated in Megatron-LM.%
        \label{tab:throughput o2}
      }
\end{table}

\subsection{Results}

Our main comparison is to (BF16, FP32) O1-level mixed precision training (\texttt{torch.amp}) as it achieves better (lower) perplexity compared to O2-level, and is natively supported by PyTorch. For throughput we also compare to O2-level, since it is considerably faster than O1-level as most of the operations are in BF16.

\textbf{Validation perplexity. } In terms of validation perplexity, we compare our BF16+SR strategy to BF16 and mixed precision training in \cref{tab:val ppl}. We observe our proposed method slightly outperforms mixed precision strategy, and significantly outperforms vanilla BF16 training. Although, our strategy does not use FP32 tensors it can converge to a better minimum thanks to being able to using a higher learning rate. In particular, we observe up to \textbf{2.6\%} improvement in perplexity (GPT-2 (350M) compared to MP); for GPT-Neo the validation perplexity decreases are more conservative (for instance, \textbf{0.6\%} on 6.7B) as the mixed precision models are already capable of achieving low perplexity. The difference in the perplexity gap would be higher if we trained for the same wall-clock time or with higher batch size to equate the memory usage.

\textbf{Throughput. } In \cref{tab:throughput o1}, we compare to O1-level mixed precision strategy and observe that our method is considerably faster. The relative speed gain gets larger as the model size increases due to increasing number of model parameters in FP32 operations, data conversions and communication for the mixed precision strategy. For instance, we see a multiplicative gain of $\mathbf{1.54\times}$ for 6.7B model.  Additionally, we compare the throughput to O2-level (almost BF16 training) in \cref{tab:throughput o2} using Megatron-LM repository \citep{shoeybi2020megatronlm} since it has more efficient training implementations than NeMo. We observe that our BF16+SR strategy is only marginally faster, as almost all operations for O2-level is done in BF16 and only accumulations are in FP32.

\textbf{Memory. } As our BF16+SR training strategy uses only BF16 tensors without any auxiliary tensors,
it is the most efficient strategy in terms of memory as can be seen in \cref{tab:memory o1}. Note that GPT-Neo experiments use tensor parallelism which shards the model across GPUs resulting in a smaller memory difference compared to TP=1. \cref{tab:memory o1} indicates a $\mathbf{30\%}$ reduction for GPT-2 (770M) and $\mathbf{21\%}$ reduction for GPT-Neo (6.7B); the memory save is so significant that the training could support considerably larger models, batch sizes, or lower gradient accumulation steps. 


\textbf{On shared randomness. } Throughout our experiments we share the randomness by keeping a designated random state for update steps as in \cref{alg}. The absence of shared randomness results in suboptimal performance due to model drift across different devices. When a large learning rate is employed without shared randomness, the SR strategy is prone to divergence. Conversely, with smaller learning rates, while divergence is mitigated, the final validation loss experiences a reduction in accuracy. The influence of shared randomness is also contingent upon the distributed computing setting. In particular, when sharding is utilized (e.g., via TP), its effect is more nuanced, as the majority of the model remains shared across devices, with the exception of components such as normalization layers and, in some cases, embedding layers. 

Overall, we observe that BF16+SR strategy is faster, more accurate, and memory efficient compared to the state-of-the-art O1 and O2 mixed precision methods. Another advantage is, it is easy to implement and does not require substantial changes to the model or training process. To summarize our findings we report the results for GPT-Neo (2.7B) in \cref{tab:overall-gpt2-770m}. We defer the training loss curves, ablation studies and more details on experiments to \cref{app:experiments}.

\begin{table}[h]
\centering
    \begin{tabular}{lcc}  
    Metric & MP & SR \\ \hline
    Validation ppl  &  10.31 & \textbf{10.22}  \\
    Throughput (o1) & 57 seq/sec & \textbf{84 seq/sec}  \\
    Throughput (o2) & 71 seq/sec & \textbf{75 seq/sec}  \\
    Memory  & 1572GB & \textbf{1371GB}  \\
  \end{tabular}
  \captionof{table}
      {
        Overall comparison of all metrics for GPT-Neo (2.7B).%
        \label{tab:overall-gpt2-770m}
      }
\end{table}




\textbf{Downstream tasks. } In this work, our main focus is optimization and pre-training of LLMs. We pre-train models up-to 50B tokens, as a result we do not expect the models to do very well in downstream tasks. However, for the sake of a downstream training comparison, here, we report zero shot evaluation experiments on common reasoning tasks using the 2.7B model. The results in \cref{tab:downstream} indicates that BF16+SR is able to match the mixed precision downstream performance.
\begin{table}[h!]
\centering
\begin{tabular}{lcc}
\toprule
{Dataset / Accuracy} & {MP} & {SR} \\
\midrule
hellaswag (acc)            & $30.21\%$ & $30.40\% $ \\
hellaswag (acc norm)       & $33.71\% $ & $34.42\% $ \\
arc-easy (acc)             & $48.19\% $ & $47.43\% $ \\
arc-easy (acc norm)        & $41.84\% $ & $42.21\% $ \\
arc-challenge (acc)        & $21.84\% $ & $22.27\% $ \\
arc-challenge (acc norm)   & $25.26\% $ & $26.71\% $ \\
openbookqa (acc)           & $17.67\% $ & $19.20\% $ \\
openbookqa (acc norm)      & $29.60\% $ & $29.66\% $ \\
piqa (acc)                 & $37.10\%$  & $36.50\%$ \\
\bottomrule
\end{tabular}
\caption{Comparison of mean zero-shot accuracies for various datasets under BF16+SR and (BF16, FP32) MP settings.}
\label{tab:downstream}
\end{table}

Note that both methods do not perform great in these tasks, as typically to excel in these tasks very long duration of pre-training (usually with number of tokens larger than 250B) or fine-tuning from an existent heavily pre-trained model is used. In the paper our main focus was pre-training/optimization with SR, where models were pre-trained for up to 50B tokens from scratch; hence, we are more interested in the perplexity/loss values similar to \citep{yu2024collage,zamirai2021revisitingbfloat16training,rae2022scalinglanguagemodelsmethods} does for quantized LLM training. Nonetheless, we conclude SR matches MP in downstream tasks in our experiments.


\section{Conclusion}

We demonstrate for the first time that low precision BF16 combined with SR training is a highly competitive approach for large-scale LLM training. Our theoretical analysis on Adam shows that, SR does not hinder convergence—its error can be absorbed within Adam's natural tolerance, particularly when an appropriately large learning rate is applied. Updates with SR effectively lead to implicit quantization-aware training. These theoretical insights are in parallel to our empirical results, where BF16+SR is able to outperform mixed-precision strategies (BF16, FP32). Overall, SR training is lightweight, straightforward to implement, and offers significant advantages, including higher throughput, reduced memory usage, lower communication/network overhead, and improved validation perplexity. We note that our method generalizes in a straightforward manner to even lower precision such as FP8/4, which we leave as future works due to limited hardware support. 

\section*{Acknowledgment}

We thank Haozheng Fan and Can Karakus for helpful discussions. We also thank Yida Wang and George Karypis for their support in AWS AI Research.

\clearpage 

\bibliographystyle{plainnat}
\bibliography{bibliography}

\newpage

\appendix 
\onecolumn 
{\centering
     \normalfont\huge
     \textbf{Appendix}\par}

\section{Implicit Regularization and Proof of Theorem~\ref{thm:implicit reg}} \label{app:implicit reg}

\begin{proof}[Proof of Lemma~\ref{lemma:xi}]
    
Gradient perturbation can be defined as the effective perturbation on the gradient that is induced by the update with SR. In other words, for gradient descent as an example, we have the following equality:
\begin{equation*}
    x - \alpha (\nabla F(x) + \xi_{\alpha}) = Q_{SR}(x - \alpha\nabla F(x))
\end{equation*}
As a result $\xi_{\alpha}$ is defined as:
\begin{equation*}
    \xi_\alpha(x):=
    \begin{cases}
       \frac{1}{\alpha}\Big(\lceil (x-\alpha\nabla F(x) \rceil - (x-\alpha \nabla F(x)) \Big), & \text{w.p.}\ \frac{(x-\alpha\nabla F(x) - \lfloor (x-\alpha\nabla F(x) \rfloor}{\Delta} \\
       \frac{1}{\alpha}\Big(\lfloor (x-\alpha\nabla F(x)\rfloor - (x-\alpha\nabla F(x)\Big) , & \text{w.p.}\ \frac{\lceil (x-\alpha\nabla F(x) \rceil - (x-\alpha\nabla F(x)}{\Delta}
    \end{cases}.
\end{equation*}

where $\Delta= \lceil x - \alpha \nabla F(x) \rceil - \lfloor x - \alpha \nabla F(x) \rfloor$ is the difference between quantization levels in terms of gradients. $\bbE[\|\xi_{\alpha}\|^2] = \frac{1}{\alpha^2}\Big(\lceil (x-\alpha\nabla F(x) \rceil - (x-\alpha \nabla F(x)) \Big)^\top \Big(\lfloor (x-\alpha\nabla F(x)\rfloor - (x-\alpha\nabla F(x)\Big)$. Consequently, when $\alpha \rightarrow 0$, after some algebra, we have $\lim_{\alpha\to0} \alpha\bbE[\|\xi_{\alpha}\|^2] = \sum_{i=1}^d \Delta_{x_i} |\nabla F(x)_i|$.




\end{proof}

\begin{theorem*}[Restatement of \cref{thm:implicit reg}.]
    Let $F(x)$ be the loss function to be minimized, $\alpha$ be a small learning rate, $\xi_{\alpha}(x_t)$ be random vector quantization error while doing updates with SR. Then, implicitly, the following expected modified loss, $F_{SR}(x)$, is being optimized:

    \begin{align*}
        F_{SR}(x) &= F(x) + \frac{\alpha}{4} \| \nabla F(x) \|^2 + \frac{\alpha}{4}\mathbb{E}\| \xi_{\alpha}(x) \|^2 \\
        & = F(x) + \frac{\alpha}{4} \| \nabla F(x) \|^2 + \frac{\alpha}{4}\sum_{i=1}^d \Big(\lceil - \nabla F(x) \rceil_{\alpha} - (-\nabla F(x))\Big)_i\Big(- \nabla F(x) - \lfloor - \nabla F(x) \rfloor_{\alpha}  \Big)_i
    \end{align*}
    
\end{theorem*}

\begin{proof}
    Our proof follows similar logic to \cite{barrett2020implicit}, but we consider the perturbation due to SR. We start by defining $\hat{l}(x) := l(x) +  {\xi}_{\alpha}(x)$ which is the update due to gradient descent when SR is applied. To measure the introduced error due to taking steps using modified-perturbed flow we need to compare the continuous updates and discrete updates. Note, the discrete updates with perturbed gradient flow results in the following modified flow:
    \begin{align*}
        \tilde{l}(x(t)) &:= \hat{l}(x(t)) + \alpha \hat l_1(x(t)) + \alpha^2 \hat{l}_2(x(t)) + \mathcal{O}(\alpha^3) \\
        &= l(x) + {\xi}_{\alpha}(x) + \alpha \hat{l}_1(x(t)) + \alpha^2 \hat{l}_2(x(t)) + \mathcal{O}(\alpha^3)
    \end{align*}
    Note that gradient descent (Euler method) is defined by $x_{t+1} = x_t + \alpha l(x_t)$; in this work we are interested in the average path the discrete updates follow, i.e., $\mathbb{E}_{\xi}x_{t+1} = x_t + \alpha \mathbb{E}_{\xi}[\tilde{l}(x_t)]$. The goal of backward error analysis is to compute $\hat{l}_i$s; here, we equate one step of gradient descent on original gradient flow with Taylor expansion via the modified-perturbed flow (in expectation). This yields,

    \begin{align*}
        x_t + \alpha l(x_t) &= \bbE_{\xi} [x_t + \alpha \tilde{l}(x_t) + \alpha^2 \nabla \tilde{l}(x_t) \tilde{l}(x_t) + \mathcal{O}(\alpha^3)]\\
        &= x(t) + \alpha {l}(x(t)) + \alpha^2(\hat{l}_1(x(t)) + \frac{1}{2} \bbE_{\xi} [\nabla ({l}(x(t))+\xi_{\alpha}(x(t))) ({l}(x(t))+\xi_{\alpha}(x(t))) + \mathcal{O}(\alpha^3)
    \end{align*}
    Assuming $ {l}(x)\neq0$ (this corresponds to case where the local minimum is achieved and derivative of $\xi$ is not continuous) we have that expectation and differentiation operators commute due to Leibniz rule. Using zero mean property of $\xi$, therefore we have,
    \begin{align*}
        x_t + \alpha l(x_t) &= x(t) + \alpha {l}(x(t)) + \alpha^2\Big(l_1(x(t)) + \frac{1}{2} \nabla ({l}(x(t)){l}(x(t)) + \frac{1}{2}\bbE[\nabla \xi_{\alpha}(x(t))) \xi_{\alpha}(x(t))]\Big) + \mathcal{O}(\alpha^3)
    \end{align*}
    Now we can compute $\hat{l}_1(x)$ as,
    \begin{align*}
        \hat{l}_1(x(t)) &= -\frac{1}{2}\nabla {f}(x(t)){f}(x(t)) - \frac{1}{2}\bbE[\nabla \xi_{\alpha}(x(t))) \xi_{\alpha}(x(t))]\Big) \\
        &= - \nabla \|\nabla F(x) \|^2 - \nabla \bbE \|\xi_{\alpha}(x) \|^2
    \end{align*}
    Then we found the average perturbed-modified flow as:
    \begin{align*}
        \bbE \dot{x} &= l(x) + \bbE \xi_{\alpha}(x) + \bbE \alpha \hat{l}_1(x) + \mathcal{O}(\alpha^2) = - \nabla F(x) - \frac{\alpha}{4} \nabla \|\nabla F(x) \|^2 - \frac{\alpha}{4} \nabla \bbE \|\xi_{\alpha}(x) \|^2 \\
        & = - \nabla \Big( F(x)+ \frac{\alpha}{4} \|\nabla F(x) \|^2 +  \frac{\alpha}{4} \bbE \|\xi_{\alpha}(x) \|^2 \Big).
    \end{align*}
 and we obtain the theorem statement. Note that with this modified flow, $\hat l_1$ is constructed such that there is a $\mathcal{O}(\alpha^3)$ error between the true solution of gradient flow and discrete update, i.e. $\|\bbE_{\xi}x_t-x(t)\|<\mathcal{O}(\alpha^3)$. This is similar to \cite{barrett2020implicit}; hence, on average SR does not introduce additional error as expected due the non-biasedness. For any biased rounding scheme,  such as nearest, $\bbE \xi_{\alpha}(x)$ in the update step cannot be cancelled, as a result there will be a $\mathcal{O}(\alpha)$ error difference between the true updates and modified updates at every iteration.

\end{proof}

\section{Adam with quantized model updates convergence} \label{app:convergence}


\subsection{Proof of Theorem~\ref{thm:convergence}}

\begin{theorem*} [Restated \cref{thm:convergence}]
    When SR is used for the update step to obtain quantized weights we have:
    \begin{align*}
    \frac{1}{T}\sum_{t=1}^T \bbE \|\nabla F(x_t) \|^2 & \leq \frac{2R(F_0-F^*)}{\alpha T} + \frac{2Rd}{T}
    \left( \frac{2R}{\sqrt{1-\beta_2}}+\frac{\alpha L}{1-\beta_2} \right) \left( \ln\left(1 + \frac{R^2}{\epsilon(1-\beta_2)}\right) + T \ln\left(\frac{1}{\beta_2}\right) \right) \\
    & + \frac{Rd \Delta L}{T\sqrt{1-\beta_2}} \left( \sqrt{T \ln\left(1 + \frac{R^2}{\epsilon(1-\beta_2)}\right)} + T \sqrt{\ln\left(\frac{1}{\beta_2}\right)} \right).
\end{align*}
\end{theorem*}

\textbf{Adaptive updates.} Similar to \cite{defossez2022a} we define the following vectors iteratively. 
\begin{align*}
    & m_{t,i} = \beta_1 m_{t-1,i} + \nabla_i f_t(x_{t-1})\\
    & v_{t,i} = \beta_2 v_{t-1,i} + \nabla_i f_t(x_{t-1})^2 
\end{align*}
Note that $\hat m_{t,i} = (1-\beta_1)m_{t,i}$ and $\hat v_{t,i} = (1-\beta_2)v_{t,i}$ give the updates in Adam. 
we use the learning rate to absorb $(1-\beta_1),(1-\beta_2)$ terms for simplicity. In particular, if one has an update $x_{t,i} = x_{t-1,i} - \alpha_t\frac{m_{t,i}}{\sqrt{\epsilon+v_{t,i}}}$; Adam is recovered with $\alpha_t = \alpha \frac{1-\beta_1}{\sqrt{1-\beta_2}}$. If we include the bias correction terms, Adam is obtained with $\alpha_t = \alpha \frac{\sqrt{1-\beta_2^t}}{\sqrt{1-\beta_2}}\frac{{1-\beta_1^t}}{{1-\beta_1}}$. For the update step we further define:
\begin{align*}
    x_{t,i} = Q_{SR}\Big(x_{t-1,i} - \alpha_t\frac{m_{t,i}}{\sqrt{\epsilon+v_{t,i}}}\Big).
\end{align*}


\begin{lemma}
    Let us define $u_{t,i}:=\frac{m_{t,i}}{\sqrt{\epsilon+v_{t,i}}}$, the perturbation due to SR updates, defined as $r_{t,i} := Q_{SR}(x_{t-1,i} - \alpha_t u_{t,i})- (x_{t-1,i} - \alpha_t u_{t,i})$ at timestep $t$ per dimension $i\in[d]$ has the following upper bound on $l_2$ norm,
    \begin{equation}
        \bbE_p\|r_{t}\|^2_2 \leq \Delta \alpha_t \|u_t\|_1 \leq \Delta \sqrt{d} \alpha_t \|u_t\|_2 
    \end{equation}
    where $\bbE_p$ is the expectation with respect to randomness of SR and $\Delta=\Delta_{x_{t-1,i} - \alpha_t u_{t,i}}$.
    \label{lem:q err}
\end{lemma}
\begin{proof}
Here we consider the update terms in Adam instead of gradient terms, other than that the proof is similar to Lemma 1 in \cite{li2017training}. We redefine the perturbation due to SR updates for ease of analysis, $r_{t,i} = Q_{SR}(x_{t-1,i} - \alpha_t\frac{m_{t,i}}{\sqrt{\epsilon+v_{t,i}}})- (x_{t-1,i} - \alpha_t\frac{m_{t,i}}{\sqrt{\epsilon+v_{t,i}}})$. Choose a random number $p \in [0,1]$, $r_{t,i}$ can be expressed as,

\begin{equation*}
    r_{t,i} = \Delta \cdot
    \begin{cases}
       \frac{-x_{t-1}+\alpha_t u_{t,i}}{\Delta} + \frac{\lfloor x_{t-1} - \alpha_t u_{t,i} \rfloor}{\Delta} + 1, &  \text{when } p\leq \frac{x_{t-1}-\alpha_t u_{t,i}}{\Delta} - \frac{\lfloor x_{t-1} - \alpha_t u_{t,i} \rfloor}{\Delta} \\
       \frac{-x_{t-1}+\alpha_t u_{t,i}}{\Delta} + \frac{\lfloor x_{t-1} - \alpha_t u_{t,i} \rfloor}{\Delta} & \text{otherwise.}
    \end{cases}
\end{equation*}
Defining $q =\frac{x_{t-1}-\alpha_t u_{t,i}}{\Delta} - \frac{\lfloor x_{t-1} - \alpha_t u_{t,i} \rfloor}{\Delta} $, note $q \in [0,1]$, we can rewrite,
\begin{equation}
    r_{t,i} = \Delta \cdot
    \begin{cases}
       -q + 1, & \text{when } p\leq q  \\
       -q & \text{otherwise. } 
    \end{cases}
\end{equation}

Consequently, $\bbE[r_{t,i}]=0$, and it can be shown that
\begin{align*}
    \bbE[r_{t,i}^2] &\leq \Delta^2 q(1-q) \\
    & \leq  \Delta^2 \min\{q,1-q\} \\
    & \leq \Delta^2 \alpha_t \frac{|u_{t,i}|}{\Delta}\\
    & = \Delta \alpha_t {|u_{t,i}|},
\end{align*}
where the last inequality is due to the fact that the update term will be at least as large as the error term. Then $\bbE\|r_{t}\|^2_2 \leq \Delta \alpha_t \|u_t\|_1 \leq \Delta \sqrt{d} \alpha_t \|u_t\|_2 $. 
\end{proof}

Let us start by defining $\tilde v_{t,i}$ which is a virtual sequence where the contribution of the last gradient is replaced by its expected value conditioned on past iterations. 
\begin{align*}
     \tilde v_{t,i} = \beta_2 v_{t-1,i} + \bbE_{t-1} \nabla_i f_t(x_{t-1})^2.
\end{align*}
 We also define the perturbed updates as follows:
\begin{equation*}
    \tilde u_{t,i} = \frac{\nabla_i f_t(x_{t-1})}{\sqrt{\epsilon +  v_{t,i}}} + \frac{r_{t,i}}{\alpha_t}.
\end{equation*}
From the smoothness of the objective function $F$, we can use the Descent Lemma to obtain that
\begin{align*}
    F(x_{t}) & \leq F(x_{t-1}) - \alpha_t \nabla F(x_{t-1})^\top \tilde u_t+ \frac{\alpha^2 L}{2}\| \tilde u_t \|^2 \\
    & = F(x_{t-1}) - \alpha_t \nabla F(x_{t-1})^\top \Big(u_t+\frac{r_t}{\alpha_t}\Big) + \frac{\alpha_t^2 L}{2}\Big\| u_t + \frac{r_t}{\alpha_t}\Big\|^2
\end{align*}
{where $r_t$ is defined in \cref{lem:q err}.} Taking the expectation of both sides with respect to sampling randomness (conditioned on $l_1, \ldots, f_{t-1}$) and SR randomness, we have:
\begin{align} 
    \bbE_{p,t-1} F(x_{t}) &\leq F(x_{t-1}) - \alpha_t \sum_{i \in [d]} \bbE_{t-1}\left[\nabla_i F(x_{t-1})\frac{\nabla_i f_t(x_{t-1})}{\sqrt{\epsilon + v_{t,i} }} \right] + \frac{\alpha_t^2 L}{2} \bbE_{p,t-1} \Big\| u_t + \frac{r_t}{\alpha_t}\Big\|^2 \notag \\
    &= F(x_{t-1}) - \alpha_t \sum_{i \in [d]} \bbE_{t-1}\left[\nabla_i F(x_{t-1})\frac{\nabla_i f_t(x_{t-1})}{\sqrt{\epsilon + v_{t,i} }} \right] + \frac{\alpha_t^2 L}{2}  \bbE_{t-1} \| u_t \|^2 + \frac{L}{2} \bbE_{p,t-1} \| r_t \|^2 \notag \\
    & \leq F(x_{t-1}) - \alpha_t \sum_{i \in [d]} \bbE_{t-1}\left[\nabla_i F(x_{t-1})\frac{\nabla_i f_t(x_{t-1})}{\sqrt{\epsilon + v_{t,i} }} \right] + \frac{\alpha_t^2 L}{2}  \bbE_{t-1} \| u_t \|^2 + \frac{\alpha_t \sqrt{d} \Delta L}{2} \bbE_{t-1} \| u_t \|_2. \label{eq:cond exp}
\end{align}
where the first equality holds due to $\bbE[r_t]=0$.
To upper bound the second term on the right hand side, we use the following lemma which is a restatement of Lemma 5.1 in \citep{defossez2022a}, we provide the proof for completeness.

\begin{lemma}[Lemma 5.1 in \citep{defossez2022a}, updates approximately follow a descent direction]\label{thm:lemma1}
    \begin{align}
        \bbE_{t-1} \left[ \nabla_i F(x_{t-1}) \frac{\nabla_i f_t(x_{t-1})}{\sqrt{\epsilon+ v_{t,i}}} \right] \geq \frac{\nabla_i F(x_{t-1})^2}{2\sqrt{\epsilon + \tilde v_{t,i}}}- 2R\bbE_{t-1} \left[ \frac{\nabla_i f_t(x_{t-1})^2}{\epsilon + v_{t,i}} \right].
    \end{align}
\end{lemma}
\begin{proof}
    For the sake of simplicity, we define $G = \nabla_i F(x_{t-1}), \ g = \nabla_i f_t(x_{t-1}), v = \bar v_{t,i}, \tilde v = \tilde v_{t,i}$. First obviously, we have
    \begin{align} \label{lemma5: total}
        \frac{Gg}{\sqrt{\epsilon+v}} = \frac{Gg}{\sqrt{\epsilon+\tilde v}} + \underbrace{\frac{Gg}{\sqrt{\epsilon+v}} - \frac{Gg}{\sqrt{\epsilon+\tilde v}}}_A.
    \end{align}
    For the first term, we use the fact that $\bbE_{t-1}[g]=G$ to obtain that
    \begin{align}\label{lemma5:first term}
        \bbE_{t-1} \left[\frac{Gg}{\sqrt{\epsilon+\tilde v}} \right]  = \frac{G^2}{\sqrt{\epsilon+\tilde v}}.
    \end{align}
    In order to bound $A$, we begin with
    \begin{align*}
     A = Gg \frac{\tilde v - v}{\sqrt{\epsilon+v}\sqrt{\epsilon+\tilde v}(\sqrt{\epsilon+v}+\sqrt{\epsilon+\tilde v})} = Gg \frac{\bbE_{t-1}g^2 - g^2}{\sqrt{\epsilon+v}\sqrt{\epsilon+\tilde v}(\sqrt{\epsilon+v}+\sqrt{\epsilon+\tilde v})},
     \end{align*}
     where the last equality follows from the fact that $\tilde v - v = \bbE_{t-1}[g^2] - g^{2}$ (see the definition of $v_{t,i}$). Hence, from the triangle inequality we have that
     \begin{align*}
     |A| & \leq \underbrace{|Gg| \frac{\bbE_{t-1}g^2}{\sqrt{\epsilon+v}(\epsilon+\tilde v)}}_{A_1} + \underbrace{|Gg| \frac{g^2}{\sqrt{\epsilon+\tilde v}(\epsilon+ v)}}_{A_2},
    \end{align*}
    where in the inequality follows from the fact that $\sqrt{\epsilon+v}+\sqrt{\epsilon+\tilde v} \geq \max(\sqrt{\epsilon+ v}, \sqrt{\epsilon+\tilde v})$. A useful fact that will be used later is the following
    \begin{align}\label{fact1}
        \forall \lambda>0, \ x,y \in \bbR \quad xy \leq \frac{\lambda x^2}{2} + \frac{y^2}{2\lambda}.
    \end{align}
    To bound $A_1$, we use \eqref{fact1} with $\lambda = \frac{\sqrt{\epsilon+\tilde v}}{2}, x=\frac{|G|}{\sqrt{\epsilon + \tilde v}}, y = |g| \frac{ \bbE_{t-1} g^2}{2\sqrt{\epsilon+\tilde v} \sqrt{\epsilon+v}}$, which yields that
    \begin{align*}
    A_1 \leq \frac{G^2}{4\sqrt{\epsilon + \tilde v}} + \frac{g^2\bbE_{t-1} [g^2]^2}{(\epsilon+\tilde v)^{3/2}(\epsilon + v)}.
    \end{align*}
    Taking the expectation and noting $\epsilon + \tilde v \geq \bbE_{t-1} [g^2]$ ensures that
    \begin{align}
        \bbE_{t-1} [A_1] \leq \frac{G^2}{4\sqrt{\epsilon+\tilde v}} + \frac{\bbE_{t-1} [g^2]}{\sqrt{\epsilon+\tilde v}} \bbE_{t-1}\left[\frac{g^2}{\epsilon+v} \right] \leq \frac{G^2}{4\sqrt{\epsilon+\tilde v}} + R \bbE_{t-1}\left[\frac{g^2}{\epsilon+v} \right] \label{eq:lem1 A1}
    \end{align}
    where the last inequality uses our boundedness assumption $\sqrt{\bbE_{t-1} [g^2]}\leq R $ and the fact that $\sqrt{\epsilon + \tilde v} \geq \sqrt{\bbE_{t-1} [g^2]}$. Similarly, for $A_{2}$, we use \eqref{fact1} with $\lambda = \frac{\sqrt{\epsilon+\tilde v}}{2 \bbE_{t-1}g^2}, x= \frac{|Gg|}{\sqrt{\epsilon+\tilde v}}, y= \frac{g^2}{\epsilon  + v}$, which gives us (we used similar bounding arguments)
    \begin{align} \label{eq:lem1 A2}
    \bbE_{t-1} [A_2] \leq \frac{G^2}{4 \sqrt{\epsilon + \tilde v}}+ R \frac{\bbE_{t-1} [g^2]}{\epsilon+v}.
    \end{align}
    Combining both upper bounds we have,
    \begin{align*}
        \bbE_{t-1} [|A|] \leq  \frac{G^2}{2\sqrt{\epsilon + \tilde v}}+ 2R\frac{\bbE_{t-1} [g^2]}{\epsilon+v} ,
    \end{align*}
    which can be equivalently written as follows
    \begin{align*}
        \bbE_{t-1} [-|A|] \geq - \left(\frac{G^2}{2\sqrt{\epsilon + \tilde v}} + 2R\frac{\bbE_{t-1} [g^2]}{\epsilon+v} \right),
    \end{align*}
    Finally, combining further with \eqref{lemma5:first term} and substituting into \eqref{lemma5: total} gives us,
    \begin{align*}
        \bbE_{t-1} \left[\frac{Gg}{\sqrt{\epsilon+v}} \right] & \geq \frac{G^2}{\sqrt{\epsilon + \tilde v}} - \left(\frac{G^2}{2\sqrt{\epsilon + \tilde v}} + 2R\frac{\bbE_{t-1} [g^2]}{\epsilon+v}\right) = \frac{G^2}{2\sqrt{\epsilon + \tilde v}} - 2R\frac{\bbE_{t-1} [g^2]}{\epsilon+v},
    \end{align*}
    which proves the desired result. 
\end{proof}
Note that for any $i\in[d]$, $\sqrt{\epsilon+\tilde v_{t,i}}\leq R \sqrt{1 + \sum_{j=0}^{t-1} \beta_2^t}
\leq \frac{R}{\sqrt{1-\beta_2}}$. 
Hence, 
\begin{align}
    \alpha_t \frac{(\nabla_i F(x_{t-1}))^2}{2\sqrt{\epsilon+\tilde v_{t,i}}} \geq \alpha \frac{\nabla_i F(x_{t-1})^2}{2R}. \label{eq:bound on denom}
\end{align}
Using \eqref{eq:bound on denom} in conjunction with Lemma~\ref{thm:lemma1} in \eqref{eq:cond exp} further yields that
\begin{align*}
    \bbE_{t-1} [F(x_{t})] \leq F(x_{t-1}) - \left(\frac{\alpha}{2R}\|\nabla F(x_{t-1})\|^2 - 2\alpha_t R \bbE_{t-1} [\| u_t \|^2]\right)+ \frac{\alpha_t^2 L}{2}\bbE_{t-1} [\| u_t \|^2]+ \frac{\alpha_t \sqrt{d} \Delta L}{2} \bbE_{t-1} [\| u_t \|_2].
\end{align*}
Summing both sides for $t=1, \ldots ,T$, using $\alpha_t \leq \frac{\alpha}{\sqrt{1-\beta_2}}$, and from the definition of $l_2$ norm and taking the overall expectation yields that
\begin{align*}
    \bbE [F(x_{T})] \leq F(x_{0}) - \frac{\alpha }{2R} \sum_{t=0}^{T-1} \bbE \|\nabla F(x_{t})\|^2+ \left(\frac{2\alpha R}{\sqrt{1-\beta_2}}+\frac{\alpha^2 L}{2(1-\beta_2)}  \right)\sum_{t=0}^{T-1} \bbE[\| u_t \|^2]+\frac{\alpha \sqrt{d} \Delta L}{2\sqrt{1-\beta_2}} \sum_{t=0}^{T-1} \bbE [\sqrt{\| u_t \|^2}].
\end{align*}
Equivalently, we can write,
\begin{align}
    \frac{\alpha}{2R}\sum_{t=0}^{T-1} \bbE \|\nabla F(x_{t})\|^2 \leq F(x_{0})-\bbE [F(x_{T})] + \left(\frac{2\alpha R}{\sqrt{1-\beta_2}}+\frac{\alpha^2 L}{2(1-\beta_2)}  \right)\sum_{t=0}^{T-1} \bbE[\| u_t \|^2]+\frac{\alpha \sqrt{d} \Delta L}{2\sqrt{1-\beta_2}} \sum_{t=0}^{T-1} \bbE [\sqrt{\| u_t \|^2}].\label{eq:after lemma 1}
\end{align}
Now, we would like to bound the last two terms on the right hand side. For this, we extend Lemma 5.2 in \citep{defossez2022a}.
\begin{lemma}\label{thm:lemma2}
    Define $b_t= \sum_{j=1}^t \beta_2^{t-j}a_j$ for $t>0$ and $b_0 = 0$, we have
    \begin{align*}
        \sum_{t=1}^T\frac{a_t}{\epsilon+b_t} \leq  \ln\left(1 + \frac{R^2}{\epsilon(1-\beta_2)}\right) + T \ln\left(\frac{1}{\beta_2}\right).
    \end{align*}
\end{lemma}
\begin{proof}
    Note $b_t > a_t \geq 0 $ and we have that $1-z \leq \exp^{-z} $ with $z = \frac{a_t}{\epsilon + b_t} < 1$. Hence, 
    \begin{align*}
        \frac{a_t}{\epsilon + b_t} & \leq \ln(\epsilon+b_t) - \ln(\epsilon + b_t - a_t)\\
        & = \ln(\epsilon+b_t) - \ln(\epsilon + \beta_2b_{t-1}) \\
        & = \ln\left(\frac{\epsilon+b_t}{\epsilon+b_{t-1}}\right) + \ln\left(\frac{\epsilon+ b_{t-1} }{\epsilon + \beta_2b_{t-1}}\right)\\
        & \leq \ln\left(\frac{\epsilon+b_t}{\epsilon+b_{t-1}}\right) + \ln\left(\max\{1,\frac{1}{\beta_2}\}\right) \\
        & = \ln\left(\frac{\epsilon+b_t}{\epsilon+b_{t-1}}\right) + \ln\left(\frac{1}{\beta_2}\right),
    \end{align*}
    where the first equality is due to definition of $b_t$. Summing the inequality above for $t = 1,2, \ldots, T$, yields that (recall that $b_0 = 0$)

    
    \begin{align*}
        \sum_{t=1}^T\frac{a_t}{\epsilon + b_t} \leq \ln\left(1 + \frac{b_T}{\epsilon}\right) + T \ln\left(\frac{1}{\beta_2}\right)
    \end{align*}
    Furthermore, we have 
    \begin{align*}
        \sum_{t=1}^T\sqrt{\frac{a_t}{\epsilon + b_t} } & \leq \sqrt{T} \sqrt{\sum_{t=1}^T\frac{a_t}{\epsilon + b_t} }\\
        & \leq \sqrt{T} \sqrt{\ln\left(1 + \frac{b_T}{\epsilon}\right) + T \ln\left(\frac{1}{\beta_2}\right)} \\
        & \leq \sqrt{T \ln\left(1 + \frac{b_T}{\epsilon}\right)} + T \sqrt{\ln\left(\frac{1}{\beta_2}\right)}
    \end{align*}
    Also note that $b_T \leq \frac{R^2}{1-\beta_2}$.
\end{proof}
Using Lemma~\ref{thm:lemma2} in \eqref{eq:after lemma 1}, i.e., $a_t = \nabla_i f_t(x_t)^2$ and $b_t=v_{t,i}=\sum_{j=1}^t \beta_2^{t-j}a_j$, and dividing both sides by $T$, and after some algebra we have,
\begin{align*}
    \frac{1}{T}\sum_{t=1}^T \bbE \|\nabla F(x_t) \|^2 & \leq \frac{2R(F_0-F^*)}{\alpha T} + \frac{2Rd}{T}
    \left( \frac{2R}{\sqrt{1-\beta_2}}+\frac{\alpha L}{1-\beta_2} \right) \left( \ln\left(1 + \frac{R^2}{\epsilon(1-\beta_2)}\right) + T \ln\left(\frac{1}{\beta_2}\right) \right) \\
    & + \frac{Rd \Delta L}{T\sqrt{1-\beta_2}} \left( \sqrt{T \ln\left(1 + \frac{R^2}{\epsilon(1-\beta_2)}\right)} + T \sqrt{\ln\left(\frac{1}{\beta_2}\right)} \right).
\end{align*}
which concludes the proof.

\subsection{Proof sketch with quantized second order moment. } For \cref{thm:convergence} analysis we considered that the only quantized part is the model weights $x_t$, which was shown to be the most important to affect the empirical performance from the ablation studies in \cref{app:experiments}. Now we present a proof sketch in the case where second order moment $v_t$ is also quantized. Let $\delta_t$ denote the quantization error in computation of $v_t$. That is, $v_{t,i} = \beta_2 v_{t-1,i} + \nabla_i f_t(x_{t-1})^2 + \delta_t$. We would like to examine what happens to main lemmas under this error. For \cref{thm:lemma1}, there is no changes in the proof until \eqref{eq:lem1 A1}. For \eqref{eq:lem1 A1} we have $\epsilon+\tilde v \geq \bbE_{t-1}[g^2] + \delta^t$, which modifies the $R$ in \eqref{eq:lem1 A1} to $R+\delta$ where we assume $\delta$ is an upper bound on $\|\delta_t\|_2$ which depends on quantization resolution. We have the identical effect on multiplicative term $R$ in \eqref{eq:lem1 A2}. For \eqref{eq:bound on denom}, we have again loosening effect, hence, $\sqrt{\epsilon+\tilde v_{t,i}}\leq \frac{R+\delta}{\sqrt{1-\beta_2}}$. We see that the aggregated effect of $\delta_t$s are modifying $R$ constants with $R+\delta$ until \cref{thm:lemma2}. Looking at \cref{thm:lemma2}, the perturbation will require us to have a term $\hat b_t = b_t + \sum_{j=1}^t \delta_t$; for the lemma to hold we need $\hat b_t - \sum_{j=1}^t \delta_t \geq 0$ which will impose an upper bound on the $\delta_t$ (hence, on resolution of quantization for this step) in terms of $v_t$, i.e. $R$ such that it's magnitude cannot dominate $b_t$. Assuming such upper bound exists and continuing with lemma we will again end up modifying the $R^2$ term to $(R+\delta)^2$.

\subsection{Proof of \cref{thm:convergence nr}}\label{app:convergence_NR}

\begin{theorem*} [Restatement of \cref{thm:convergence nr}]

\begin{align*}
    \frac{1}{T}\sum_{t=1}^T  \bbE \|\nabla F(x_t) \|^2 & \leq \frac{2R(F_0-F^*)}{\alpha T} + \frac{2Rd}{T}
    \left( \frac{2R}{\sqrt{1-\beta_2}}+\frac{\alpha L}{1-\beta_2} \right) \left( \ln\left(1 + \frac{R^2}{\epsilon(1-\beta_2)}\right) + T \ln\left(\frac{1}{\beta_2}\right) \right) \\
    & + \frac{2Rd \Delta L}{T\sqrt{1-\beta_2}} \left( \sqrt{T \ln\left(1 + \frac{R^2}{\epsilon(1-\beta_2)}\right)} + T \sqrt{\ln\left(\frac{1}{\beta_2}\right)} \right)+ \frac{\sqrt{1-\beta_2}d(R\Delta + L\Delta^2)}{\alpha}.
\end{align*}

\end{theorem*}

\begin{proof}

The analysis is similar to stochastic rounding except $r$ is not random and $\|r_t\|_2\leq \sqrt{d}\Delta$.


From the smoothness of the objective function $F$, we can use the Descent Lemma to obtain that
\begin{align*}
    F(x_{t}) & \leq F(x_{t-1}) - \alpha_t \nabla F(x_{t-1})^\top \tilde u_t+ \frac{\alpha^2 L}{2}\| \tilde u_t \|^2 \\
    & = F(x_{t-1}) - \alpha_t \nabla F(x_{t-1})^\top \Big(u_t+\frac{r_t}{\alpha_t}\Big) + \frac{\alpha_t^2 L}{2}\Big\| u_t + \frac{r_t}{\alpha_t}\Big\|^2
\end{align*}

Taking the expectation of both sides with respect to sampling randomness (conditioned on $l_1, \ldots, f_{t-1}$), we have:
\begin{align*} 
    \bbE_{t-1} & F(x_{t}) \leq F(x_{t-1}) - \alpha_t \sum_{i \in [d]} \bbE_{t-1}\left[\nabla_i F(x_{t-1})\frac{\nabla_i f_t(x_{t-1})}{\sqrt{\epsilon + v_{t,i} }} \right] - \nabla F(x_{t-1})^\top r_t + \frac{\alpha_t^2 L}{2} \bbE_{t-1} \Big\| u_t + \frac{r_t}{\alpha_t}\Big\|^2 \\
    &  = F(x_{t-1}) - \alpha_t \sum_{i \in [d]} \bbE_{t-1}\left[\nabla_i F(x_{t-1})\frac{\nabla_i f_t(x_{t-1})}{\sqrt{\epsilon + v_{t,i} }} \right] - \nabla F(x_{t-1})^\top r_t + \frac{\alpha_t^2 L}{2}  \bbE_{t-1} \| u_t \|^2 + \alpha_tL \bbE_{t-1} [u_t^\top r_t] + \frac{L}{2} \| r_t \|^2 \\
    & \leq F(x_{t-1}) - \alpha_t \sum_{i \in [d]} \bbE_{t-1}\left[\nabla_i F(x_{t-1})\frac{\nabla_i f_t(x_{t-1})}{\sqrt{\epsilon + v_{t,i} }} \right] - \nabla F(x_{t-1})^\top r_t + \frac{\alpha_t^2 L}{2}  \bbE_{t-1} \| u_t \|^2 + \alpha_tL\bbE_{t-1}[\|u_t\|_2]\|r_t\| + \frac{L}{2} \| r_t \|^2 \\
    & \leq F(x_{t-1}) - \alpha_t \sum_{i \in [d]} \bbE_{t-1}\left[\nabla_i F(x_{t-1})\frac{\nabla_i f_t(x_{t-1})}{\sqrt{\epsilon + v_{t,i} }} \right] + \|\nabla F(x_{t-1}) \|\|r_t \| + \frac{\alpha_t^2 L}{2}  \bbE_{t-1} \| u_t \|^2 + \alpha_tL\bbE_{t-1}[\|u_t\|_2]\|r_t\| + \frac{L}{2} \| r_t \|^2. 
\end{align*}

Using the fact that $\|r_t\| \leq \sqrt{d}\Delta$ and $l_\infty$ bound assumption, we have:

\begin{align} \label{eq:rne err}
    \bbE_{t-1} F(x_{t}) \leq F(x_{t-1}) - \alpha_t \sum_{i \in [d]} \bbE_{t-1}\left[\nabla_i F(x_{t-1})\frac{\nabla_i f_t(x_{t-1})}{\sqrt{\epsilon + v_{t,i} }} \right] + Rd\Delta  + {\alpha_t^2 L}  \bbE_{t-1} \| u_t \|^2 + \alpha_tL\sqrt{d}\Delta \bbE_{t-1}\|u_t\|_2+ Ld\Delta^2.
\end{align} 

Remark that the last term in \eqref{eq:rne err} is due to variance of the nearest rounding and other error terms are due to bias. We observe that although variance is low ($\mathcal{O}(\Delta^2)$ dependence) the bias terms result in higher error compared to the proof with SR. Continuing as in proof of \cref{thm:convergence} we obtain the end result:

\begin{align*}
    \frac{1}{T}\sum_{t=1}^T \bbE \|\nabla F(x_t) \|^2 & \leq \frac{2R(F_0-F^*)}{\alpha T} + \frac{2Rd}{T}
    \left( \frac{2R}{\sqrt{1-\beta_2}}+\frac{\alpha L}{1-\beta_2} \right) \left( \ln\left(1 + \frac{R^2}{\epsilon(1-\beta_2)}\right) + T \ln\left(\frac{1}{\beta_2}\right) \right) \\
    & + \frac{2Rd \Delta L}{T\sqrt{1-\beta_2}} \left( \sqrt{T \ln\left(1 + \frac{R^2}{\epsilon(1-\beta_2)}\right)} + T \sqrt{\ln\left(\frac{1}{\beta_2}\right)} \right)+ \frac{\sqrt{1-\beta_2}d(R\Delta + L\Delta^2)}{\alpha}.
\end{align*}

\end{proof}



\section{Additional Details and Results for Experiments} \label{app:experiments}

\subsection{Training hyper-parameters and details. }

For all our experiments we use AWS \textit{p4d.24xlarge} instances each with 8 A100-40GB GPUs. For GPT-2 (350M, 770M) we train, on a single node, for 49B and 46B tokens with 12 and 7 micro batch sizes and 480 and 448 global batch sizes respectively. For GPT-Neo (1.3B), we train for 40B tokens with 1024 global batch size and 16 micro batch sizes on 4 nodes; for GPT-Neo (2.7B, 6.7B) we train for 20B tokens with 512 global batch size and 8 micro batch sizes on 8 nodes. For GPT-2 we do not use tensor parallelism (TP), and for GPT-Neo models we use TP=8.

\textbf{Tuning the learning rate. } For GPT-2 models we use 2000 warm-up iterations and for GPT-Neo 200 iterations as recommended. We utilize cosine scheduling. We tune the maximum learning rate individually for each method starting from the recommended/default learning rates, we leave the minimum learning rate as default. For GPT-2 models we search the learning rate in $\{2,3,4,5,6,7,8\}\times 10^{-4}$. For GPT-Neo (1.2B,2.7B) we search in 
$\{1.6,3.2,4,5\}\times 10^{-4}$, and GPT-Neo (6.7B) $\{1.2,3.6,5.5\}\times 10^{-4}$.

\subsection{Ablation studies. }

\textbf{Toy example in \cref{sec:lr with sr}}

\begin{figure}[h]
\begin{minipage}{.5\textwidth}
        \centering
    \includegraphics[width=1.0\linewidth]{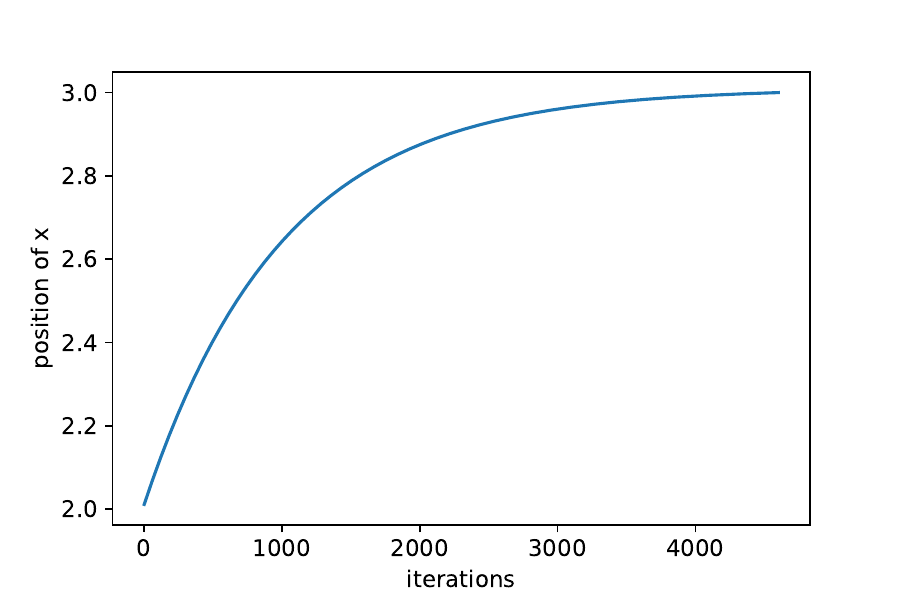}
    \end{minipage}
\begin{minipage}{.5\textwidth}
        \centering
        \includegraphics[width=1.0\linewidth]{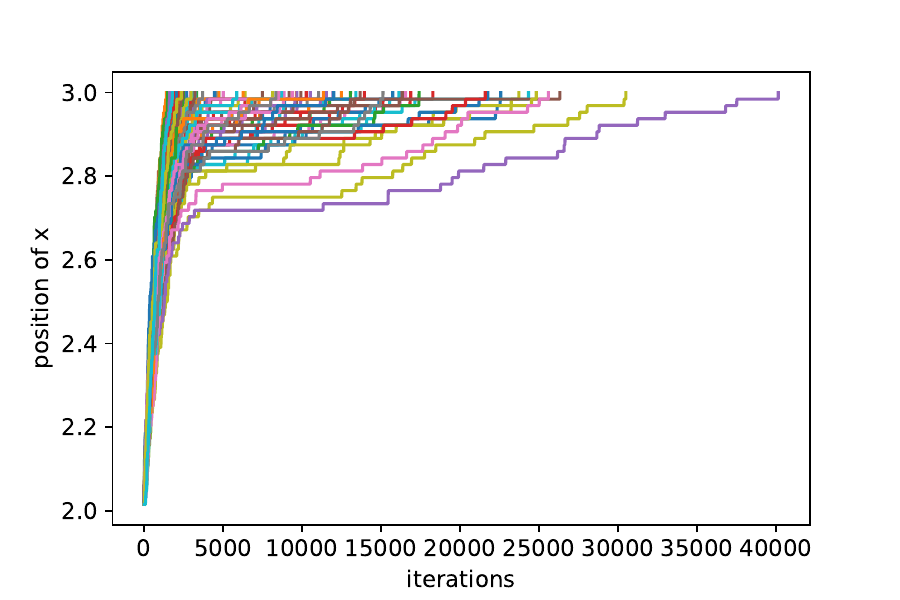}
\end{minipage}    
    \caption{Position of $x$ vs. number of iterations for FP32 updates (Left) and BF16+SR updates (Right), each position curve represents the position vs iterations for an individual run. With the decaying updates SR can take a longer time to converge, here FP32 converges in 4600 steps and SR converges in on average 7700 (over 100 runs). We observe that an individual run with SR can converge in as much as 40,000 iterations.} \label{fig:toy example} 
\end{figure}

\textbf{Effect of full precision optimizer states and gradients.}

\begin{figure}[h]
\begin{minipage}{.5\textwidth}
        \centering
    \includegraphics[width=1.0\linewidth]{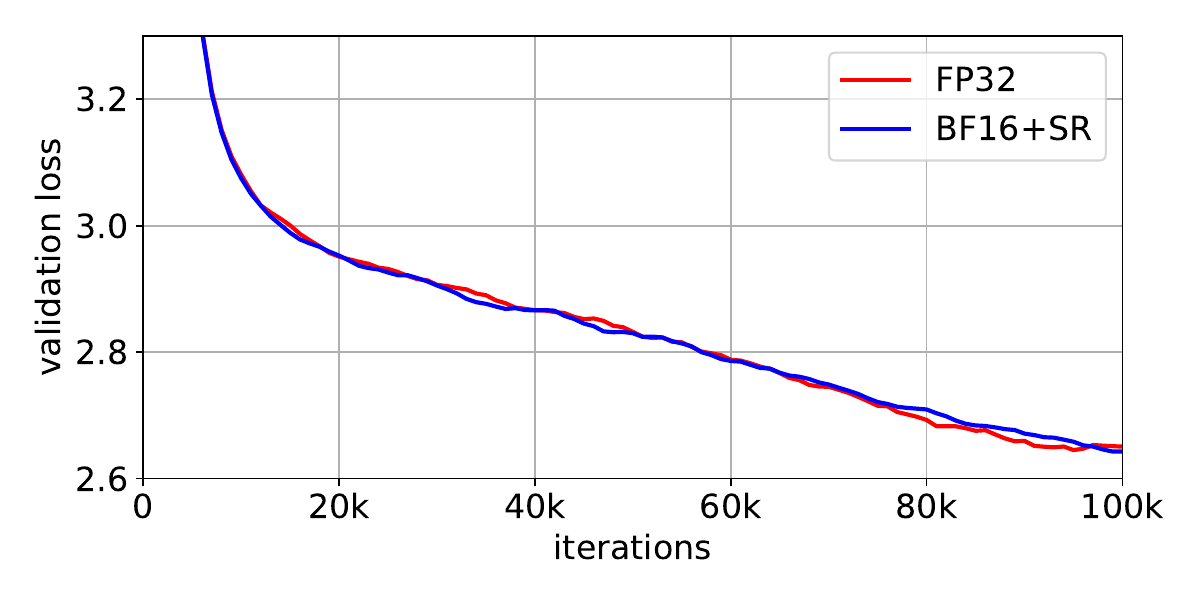}
    \end{minipage}
\begin{minipage}{.5\textwidth}
        \centering
        \includegraphics[width=1.0\linewidth]{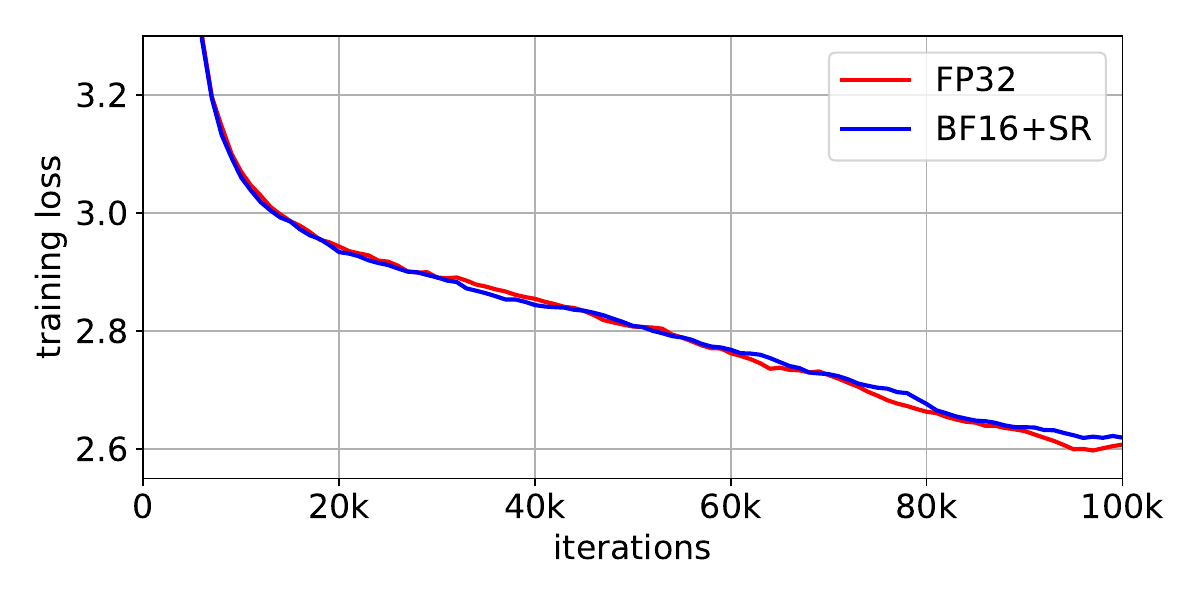}
\end{minipage}    
    \caption{Training and validation loss curves comparing FP32 training and BF16+SR strategy while training GPT-2 (350M).}\label{fig:fp32 vs bf16sr}
\end{figure}

In \cref{fig:fp32 vs bf16sr} we ablate the precisions of optimizer states, gradients and the update. In particular, every variable is in BF16 in our experiments for \cref{alg}; here, the ablation indicates that the effect of keeping optimizer states and gradients in higher precision does not affect the performance significantly. 

\subsection{Loss curves. }

\begin{figure}[h]
\begin{minipage}{.33\textwidth}
        \centering
        \includegraphics[width=1.0\linewidth]{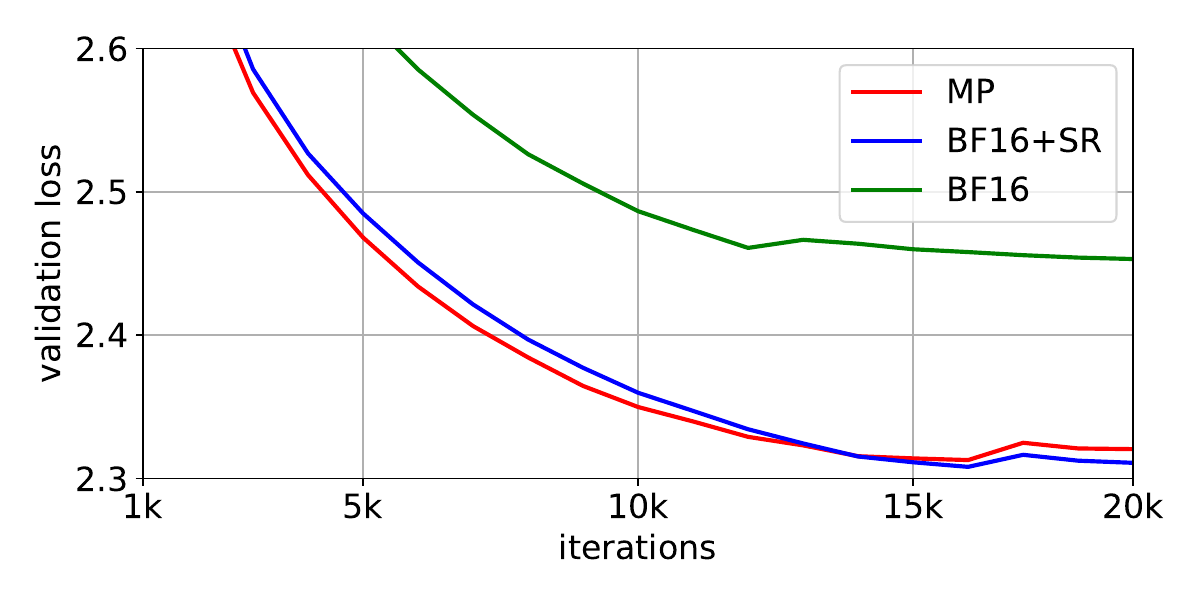}
\end{minipage}    
\begin{minipage}{.33\textwidth}
        \centering
    \includegraphics[width=1.0\linewidth]{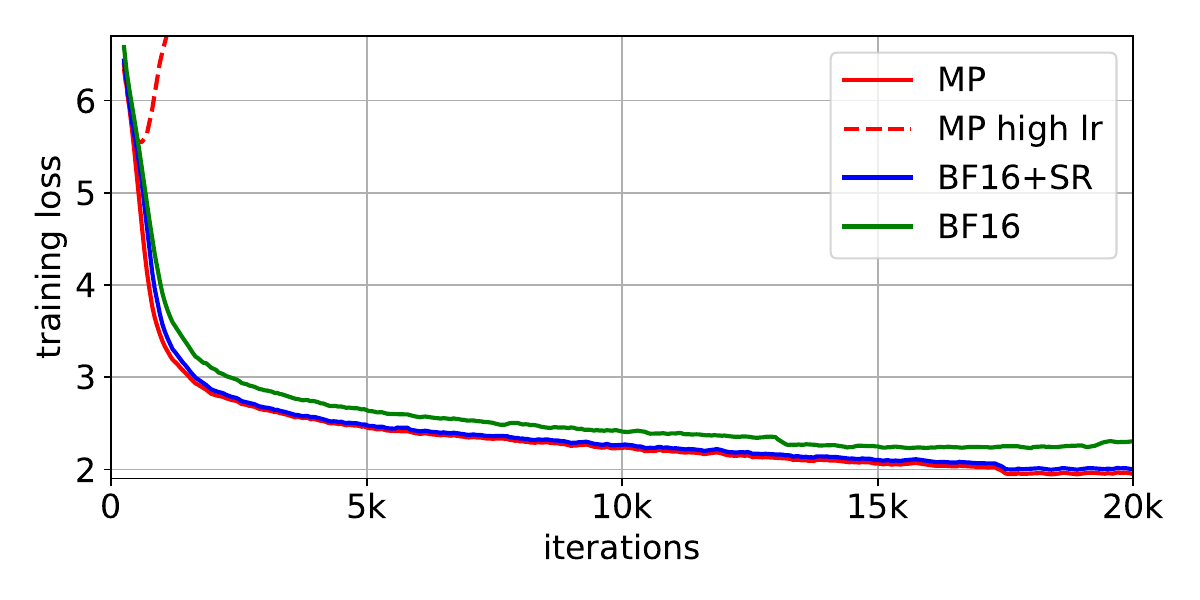}
    \end{minipage}
\begin{minipage}{.33\textwidth}
        \centering
    \includegraphics[width=1.0\linewidth]{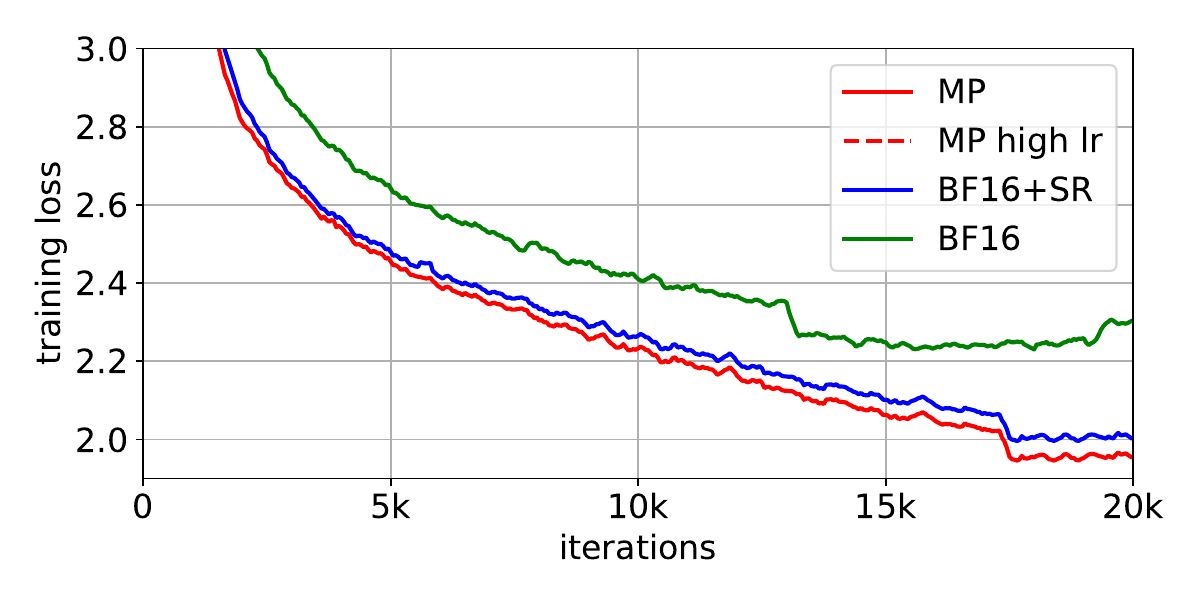}
    \end{minipage}
    \caption{Validation (left) and training losses (middle--zoomed out, right--zoomed in) for training GPT-Neo (6.7B).}\label{fig:6-7b}
\end{figure}

\begin{figure}[h]
\begin{minipage}{.5\textwidth}
        \centering
    \includegraphics[width=1.0\linewidth]{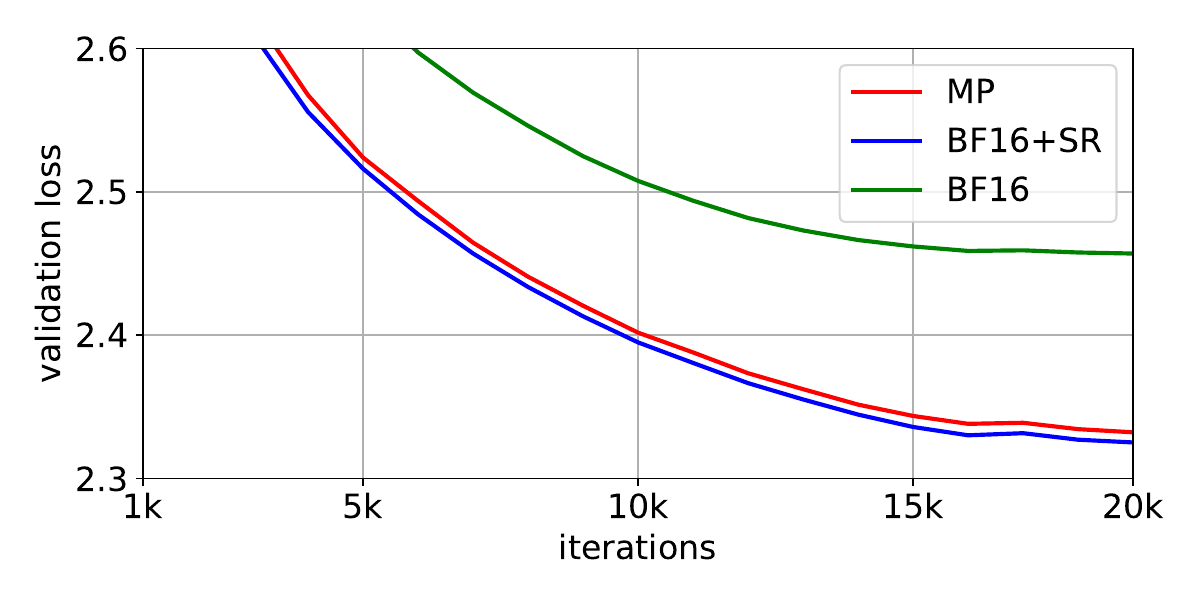}
    \end{minipage}
\begin{minipage}{.5\textwidth}
        \centering
    \includegraphics[width=1.0\linewidth]{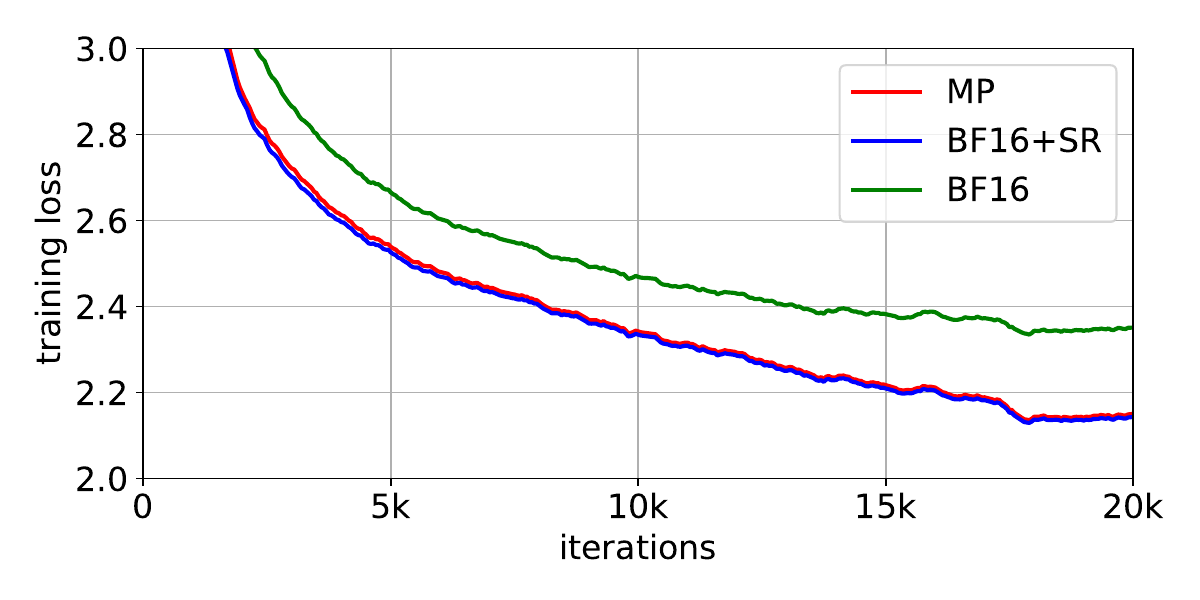}
    \end{minipage}
    \caption{Validation (left) and training losses (right) for training GPT-Neo (2.7B).}\label{fig:2-7b}
\end{figure}

\begin{figure}[h]
\begin{minipage}{.5\textwidth}
        \centering
    \includegraphics[width=1.0\linewidth]{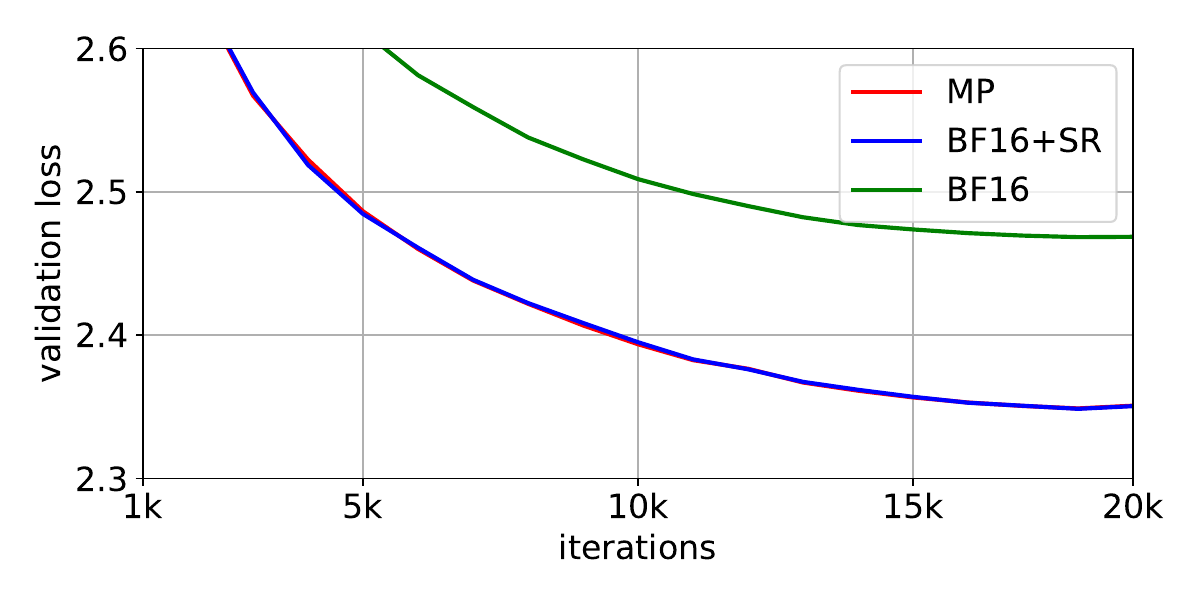}
    \end{minipage}
\begin{minipage}{.5\textwidth}
        \centering
    \includegraphics[width=1.0\linewidth]{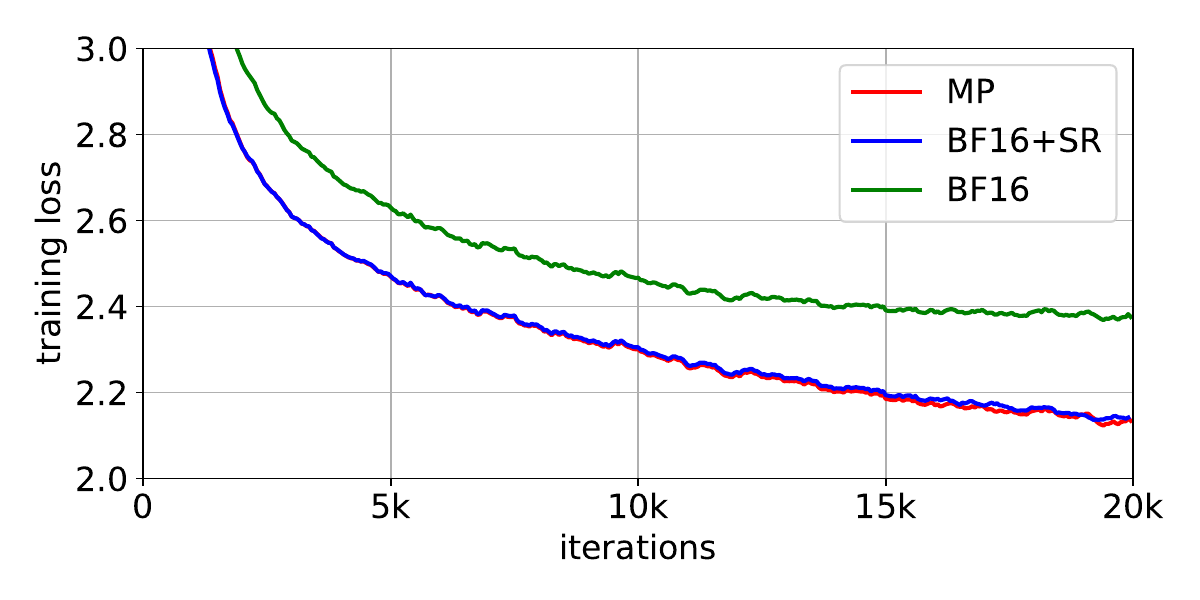}
    \end{minipage}
    \caption{Validation (left) and training losses (right) for training GPT-Neo (1.3B).}\label{fig:1-3b}
\end{figure}

\begin{figure}[h]
\begin{minipage}{.5\textwidth}
        \centering
    \includegraphics[width=1.0\linewidth]{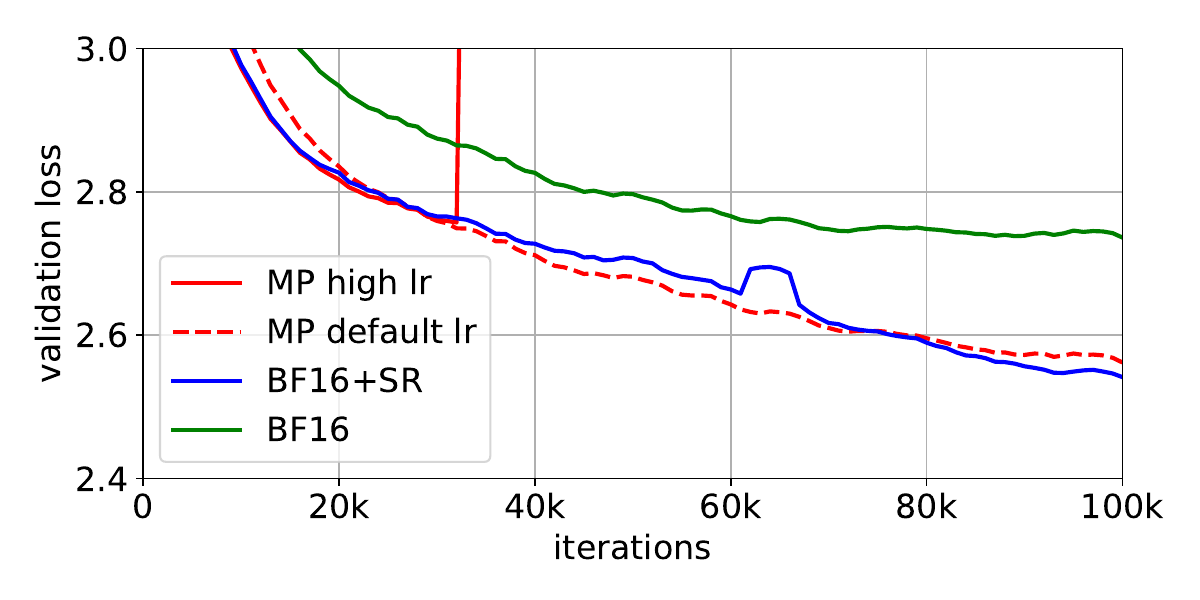}
    \end{minipage}
\begin{minipage}{.5\textwidth}
        \centering
    \includegraphics[width=1.0\linewidth]{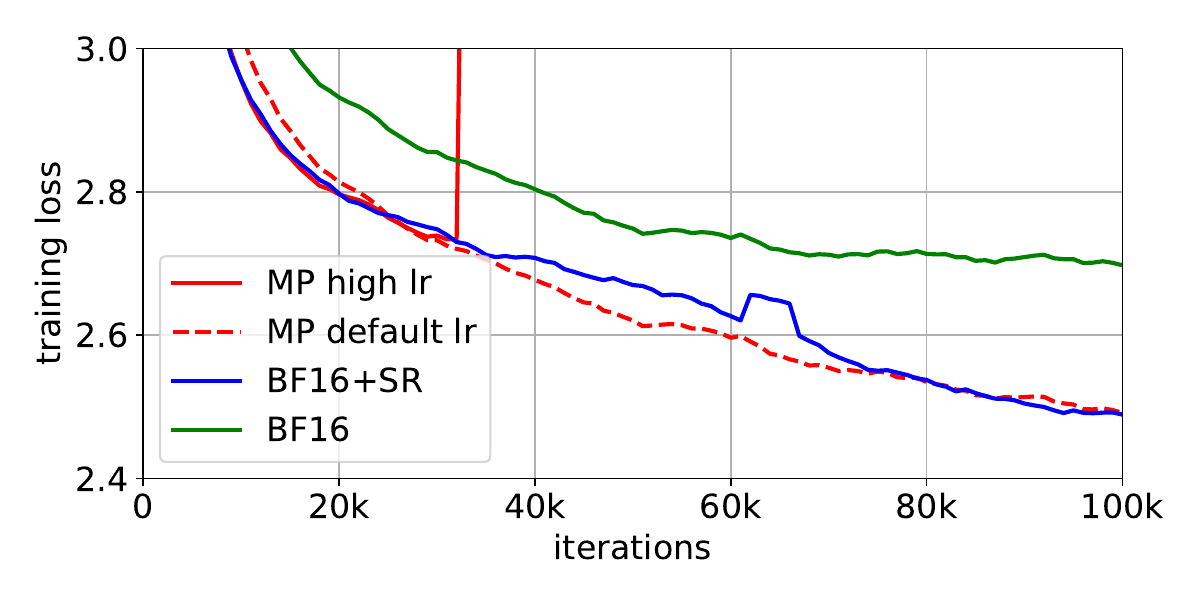}
    \end{minipage}
    \caption{Validation (left) and training losses (right) for training GPT-2 (770M).}\label{fig:770m}
\end{figure}

The convergence plots show BF16+SR training can converge to a better local minimum in terms of validation perplexity. Note that BF16+SR further gain can be provided when we equate the wall clock running times.

\end{document}